\documentclass{article}

 \usepackage[preprint]{neurips_2026}

\usepackage[utf8]{inputenc} 
\usepackage[T1]{fontenc}    
\usepackage{hyperref}       
\usepackage{url}            
\usepackage{booktabs}       
\usepackage{amsfonts}       
\usepackage{nicefrac}       
\usepackage{microtype}      
\usepackage{xcolor}         

\usepackage{amsmath}
\usepackage{amsthm}
\usepackage{amsfonts}
\usepackage{amssymb}
\usepackage{enumitem}
\usepackage{tcolorbox}
\usepackage{algorithm}
\usepackage{algorithmic}
\usepackage{colortbl}
\usepackage{multirow}
\usepackage{wrapfig}
\usepackage{caption}
\usepackage{todonotes}

\newcommand{\E}{\mathbb{E}}
\newcommand{\KL}{D_{\mathrm{KL}}}

\newtheorem{proposition}{Proposition}

\title{Learning from Language Feedback via\\Variational Policy Distillation}

\author{%
  Yang Li \quad Erik Nijkamp \quad Semih Yavuz \quad Shafiq Joty \\
  Salesforce AI Research \\
  \texttt{\{yli2, erik.nijkamp, syavuz, sjoty\}@salesforce.com}
}

\begin{document}

\maketitle

\begin{abstract}
Reinforcement learning from verifiable rewards (RLVR) suffers from sparse outcome signals, creating severe exploration bottlenecks on complex reasoning tasks. Recent on-policy self-distillation methods attempt to address this by utilizing language feedback to generate dense, token-level supervision. However, these approaches rely on a fixed, passive teacher to interpret the feedback. As the student policy improves, the teacher's zero-shot assessment capabilities plateau, ultimately halting further learning. To overcome this, we propose Variational Policy Distillation (VPD), a framework that formalizes learning from language feedback as a Variational Expectation-Maximization (EM) problem. VPD co-evolves both policies: in the E-step, the teacher is actively refined on trajectory outcomes via an adaptive trust-region update, translating textual feedback into a dynamically improved target token distribution. In the M-step, the student internalizes this dense distributional guidance on its own on-policy rollouts. By continuously improving the teacher's ability to extract actionable signals from textual critique, VPD overcomes the limitations of passive distillation. Evaluated across diverse sources of diagnostic feedback on scientific reasoning and code generation tasks, VPD consistently outperforms both standard RLVR and existing self-distillation baselines. Finally, by stress-testing our framework on rigid mathematical reasoning and cold-start regimes, we illuminate the fundamental bounds of feedback-driven self-distillation compared to pure environment-driven RL.
\end{abstract}

\section{Introduction}\label{sec:intro}

Recent leaps in the reasoning capabilities of large language models (LLMs) have been largely driven by reinforcement learning from verifiable rewards (RLVR) \citep{guo2025deepseek,shao2024deepseekmath,yang2025qwen3}. By optimizing models against objective, outcome-based correctness, RLVR avoids the high cost of human preference data in standard RLHF \citep{ouyang2022training}. However, standard policy gradient methods like GRPO \citep{shao2024deepseekmath} and their variants \citep{zheng2025group,yu2025dapo} rely almost entirely on sparse, binary outcome signals. This creates a severe credit assignment bottleneck: a minor arithmetic mistake in a complex derivation receives the same zero-reward as a completely nonsensical hallucination. Consequently, standard outcome-based RL is notoriously \emph{sample inefficient} \citep{zhang2025improving,zheng2025act}. On hard problems where the model's initial success rate is near zero, on-policy algorithms face an extreme exploration bottleneck: they receive zero positive learning signal regardless of how many rollouts are sampled, entirely wasting the valuable latent information embedded in near-miss trajectories \citep{qu2026pope,setlur2026reuse}.

To overcome this sparsity, a promising paradigm is to learn directly from \emph{language feedback}. In many real-world and agentic settings, failure is accompanied by rich textual diagnostics, such as automated critiques from a stronger LLM, compiler error traces, or user corrections. This textual feedback can potentially provide exactly the dense, localized supervision that scalar rewards lack, pointing out not just \emph{that} an attempt failed, but \emph{why} and \emph{how} it should be fixed. 

Leveraging this rich feedback effectively, however, remains an open challenge. \emph{Off-policy} methods, such as supervised fine-tuning on expert traces or feedback-revised trajectories, suffer from distribution mismatch: the student model often lacks the internal capacity to faithfully reproduce the external teacher's reasoning, leading to copycat behavior without genuine comprehension \citep{liu2023chain,scheurer2023training}. Recently, \emph{on-policy self-distillation} methods like SDPO \citep{hubotter2026reinforcement} and OPSD \citep{zhao2026self} condition the model itself on language feedback to act as a ``self-teacher,'' distilling feedback-informed next-token predictions back into the unconditioned policy. Crucially, by sampling directly from the student's own distribution, these methods avoid the severe train-inference distribution mismatch of off-policy approaches \citep{shenfeld2026self,agarwal2024policy}, allowing the teacher to act as an internal critic providing dense, token-level learning signals \citep{hubotter2026reinforcement}.

Yet, existing self-distillation approaches suffer from a critical flaw: they treat the feedback-conditioned self-teacher as a \emph{fixed, passive function}. The quality of the distillation signal depends entirely on the model's zero-shot ability to parse and exploit the language feedback. If the critique is noisy, or if the model cannot yet map natural language hints to structural token adjustments, the self-teacher's guidance can become counterproductive. Furthermore, as the student internalizes basic corrections, the zero-shot advantage of appending feedback diminishes. Since a passive teacher is never explicitly trained to be a sharper critic, its ability to distinguish between increasingly subtle reasoning errors plateaus, ultimately starving the student of further meaningful gradients.

To address this, we propose \textbf{Variational Policy Distillation (VPD)}, a principled framework that frames learning from language feedback as a variational inference problem. Instead of taking the self-teacher's feedback interpretation for granted, VPD treats the feedback-conditioned model as an \emph{approximate posterior} over correct solutions that must be actively optimized alongside the student policy. This variational perspective naturally yields an Expectation-Maximization (EM) algorithm that enables the teacher and student to co-evolve:

\begin{itemize}[leftmargin=*,topsep=0pt,itemsep=0pt]
    \item \textbf{E-step (Teacher Refinement):} We actively train the teacher's ability to interpret language feedback. By optimizing the teacher to distinguish between successful and failed trajectories \emph{given} the rich textual critique, we effectively teach the teacher how to read and leverage the feedback.
    \item \textbf{M-step (Student Optimization):} We distill this refined knowledge back into the student. By minimizing the token-level KL divergence against the improved teacher on its own on-policy rollouts, the student internalizes this dense learning signal to succeed zero-shot at deployment.
\end{itemize}

By ensuring the teacher's assessment capabilities scale alongside the student's reasoning, VPD extracts significantly more value from language feedback than passive distillation methods. We summarize our contributions as follows:
\begin{enumerate}[leftmargin=*,topsep=0pt,itemsep=0pt]
    \item We formalize on-policy learning from language feedback as a Variational EM procedure. This introduces an explicit, feedback-aware teacher update (absent from prior self-distillation methods) implemented via unpaired preference optimization. By dynamically anchoring the teacher to the current student, we enforce an adaptive trust region that ensures highly stable on-policy KL distillation over a shared-weight network.
    \item We present a comprehensive empirical study instantiating VPD across three sources of diagnostic feedback: deterministic environment verifiers, contrastive sibling rollouts, and autonomous self-critique. Evaluated on benchmarks spanning competitive programming and scientific reasoning, VPD consistently outperforms standard RLVR and self-distillation baselines.
    \item We characterize the fundamental regimes in which language-feedback distillation outperforms sparse RL, and where it does not. By stress-testing our framework on base-model cold-start scenarios and challenging mathematical reasoning, we empirically establish the bounds of language-driven self-distillation. While VPD significantly mitigates and delays the training collapse typically observed in these settings, our results demonstrate that pure sparse RL ultimately remains the most effective paradigm.
\end{enumerate}

\section{Preliminaries}\label{sec:prelim}
\vspace{-5pt}

\textbf{Reinforcement Learning from Verifiable Rewards.}
We model language generation as a contextual bandit problem where the context is the user prompt $x \in \mathcal{X}$. A language model, parameterized by $\theta$, represents a policy $\pi_\theta$ that generates a response $y = (y_1, \ldots, y_T)$ autoregressively: $y \sim \pi_\theta(\cdot \mid x)$. In the RLVR framework, an environment or rule-based verifier assesses the final response and assigns a scalar outcome reward $r(x, y)$. For complex reasoning tasks, such as mathematical theorem proving or competitive programming, this reward is typically sparse and binary, $r(x, y) \in \{0, 1\}$.

The standard RLVR objective seeks to maximize the expected reward while penalizing deviations from an initial reference policy $\pi_{\text{ref}}$ (typically the supervised fine-tuned model) to prevent catastrophic forgetting or "over-optimization" \citep{gao2023scaling}:
\begin{equation}\label{eq:rlvr}
    \mathcal{J}_{\text{RLVR}}(\theta) = \E_{x \sim \mathcal{D}, y \sim \pi_\theta(\cdot|x)} \left[ r(x, y) \right] - \beta \KL \left( \pi_\theta(\cdot \mid x) \parallel \pi_{\text{ref}}(\cdot \mid x) \right),
\end{equation}
where $\beta$ controls the strength of the KL penalty. Modern post-training pipelines often optimize this objective using algorithms like Group Relative Policy Optimization (GRPO) \citep{shao2024deepseekmath} and its variants \citep{zheng2025group,yu2025dapo}, which estimate gradients using advantage scores normalized across a group of responses. However, because $r(x,y)$ is a sparse outcome signal, if the model fails to sample any correct answers for a given prompt (i.e., $r(x, y_i) = 0$ for all $i$), the advantage scores collapse, halting the learning process and establishing a severe exploration bottleneck \citep{qu2026pope}.

\textbf{On-Policy Self-Distillation.}
To circumvent the sparsity of outcome-based rewards, recent approaches leverage rich textual feedback $\mathcal{C}$ (e.g., compiler error messages or LLM-generated critique) to construct a dense learning signal. Methods like Self-Distillation Policy Optimization (SDPO) \citep{hubotter2026reinforcement} condition the model itself on $\mathcal{C}$ to act as an on-policy ``self-teacher''.

Given a student rollout $y \sim \pi_\theta(\cdot \mid x)$ and its corresponding feedback $\mathcal{C}$, the self-teacher is defined as the identical model conditioned on the augmented prompt: $\pi_\theta(\cdot \mid x, \mathcal{C})$. The goal is to align the unconditioned student policy with the feedback-informed teacher's next-token distribution. The SDPO objective minimizes the token-level KL divergence on the student's own rollouts:
\begin{equation}\label{eq:sdpo}\textstyle
    \mathcal{L}_{\text{SDPO}}(\theta) = \E_{x \sim \mathcal{D}, y \sim \pi_\theta(\cdot|x)} \left[ \sum_{t=1}^T \KL \left( \pi_\theta(\cdot \mid x, y_{<t}) \parallel \text{sg}\left[ \pi_\theta(\cdot \mid x, \mathcal{C}, y_{<t}) \right] \right) \right],
\end{equation}
where $\text{sg}[\cdot]$ denotes the stop-gradient operator. Note that Forward KL or JS divergence can interchangeably be used here depending on desired dynamics. While Eq.~\ref{eq:sdpo} provides dense gradients, the stop-gradient highlights a fundamental limitation: the teacher is never explicitly optimized. Instead, the teacher $\pi_\theta(\cdot \mid x, \mathcal{C})$ operates purely zero-shot, relying on its pre-existing capacity to interpret the textual feedback $\mathcal{C}$. Since the teacher is not trained to refine its diagnostic interpretation, this creates a ceiling effect that restricts the gradients the teacher can ultimately provide to an improving student.

\section{Variational Policy Distillation}\label{sec:method}

In contrast to passive self-distillation methods that treat the feedback-conditioned model as a fixed heuristic, we frame learning from language feedback as a variational inference problem. This perspective allows the teacher to co-evolve alongside the student, actively learning to extract deeper insights from the textual feedback. 

\subsection{Variational Formulation}

As established in the literature \citep{peng2019advantage,rafailov2023direct,go2023aligning}, the optimal policy $\pi^*$ under the KL-regularized RLVR objective (Eq.~ \ref{eq:rlvr}) takes the form of a reward-tilted distribution:
\begin{equation}\textstyle
    \pi^*(y \mid x) = \frac{1}{Z(x)} \pi_{\text{ref}}(y \mid x) \exp\left( \frac{1}{\beta} r(x, y) \right).
    \label{eq:optimal_policy}
\end{equation}
Theoretically, optimizing the original RLVR objective is mathematically equivalent to minimizing the reverse KL divergence $\KL(\pi_\theta \parallel \pi^*)$ (see Appendix~\ref{sec:theory_app} for full derivations). In practice, however, directly minimizing this divergence is computationally infeasible because the partition function $Z(x)$ is analytically intractable, preventing us from explicitly evaluating the target distribution $\pi^*$. Standard reinforcement learning methods bypass this intractability by taking the gradient of the objective, which elegantly cancels out $Z(x)$ and yields standard policy gradient estimators. Yet, relying on these gradients inherently reduces the optimization back to sampling-based reward estimation, thrusting us right back into the sparse reward bottleneck discussed in Sec.~\ref{sec:prelim}.

To bypass this intractability, we cast the alignment process as a variational inference problem. We introduce a parameterized teacher network, $q_\phi(y \mid x, \mathcal{C})$—conditioned on the dense diagnostic feedback $\mathcal{C}$—to serve as a tractable approximate posterior for the optimal distribution $\pi^*$. While the unconditioned student $\pi_\theta$ must blindly search the vast trajectory space for sparse rewards, the inclusion of $\mathcal{C}$ allows the teacher $q_\phi$ to more effectively approximate the high-reward modes of $\pi^*$.

Mathematically, introducing this surrogate allows us to lower-bound the intractable RLVR objective using an Evidence Lower Bound (ELBO) (see Appendix \ref{app:em_bound} for details). This variational formulation naturally decomposes the training process into an Expectation-Maximization (EM) algorithm:
\begin{enumerate}[leftmargin=*,topsep=0pt,itemsep=0pt]
    \item \textbf{E-Step (Teacher Refinement):} We optimize the teacher parameters $\phi$ to minimize its divergence from the reward-tilted optimal target: $\min_\phi D_{\text{KL}}(q_\phi(y \mid x, \mathcal{C}) \parallel \pi^*(y \mid x))$. This forces the teacher to actively learn how to translate textual diagnostics into high-reward token distributions.
    \item \textbf{M-Step (Student Distillation):} We update the student parameters $\theta$ to minimize its divergence from the refined teacher: $\min_\theta D_{\text{KL}}(\pi_\theta(y \mid x) \parallel q_\phi(y \mid x, \mathcal{C}))$. This allows the student to internalize the teacher's dense diagnostic guidance.
\end{enumerate}
Crucially, while we maintain distinct notation for the teacher ($\phi$) and student ($\theta$) to mathematically isolate their alternating optimization phases, both policies are instantiated within a single, shared-weight neural network ($\phi = \theta$) in practice. The two distributions remain behaviorally distinct simply because the teacher is conditionally prompted with the diagnostic feedback $\mathcal{C}$. This unified architecture allows us to execute complex co-evolutionary distillation while entirely eliminating the memory overhead typically associated with dual-model paradigms. 

\vspace{-2pt}
\subsection{E-Step: Teacher Refinement via Off-policy Preference Optimization}
\label{subsec:e_step}
\vspace{-2pt}

For the student $\pi_\theta$ to learn effectively during the subsequent M-step, the teacher $q_\phi$ must first become a highly accurate surrogate for the intractable optimal policy $\pi^*$. Therefore, the primary objective of the E-step is to minimize the divergence $\KL(q_\phi \parallel \pi^*)$. As formally derived in Appendix \ref{app:em_bound}, we can decompose this divergence as follows:
\begin{equation}\textstyle
    \KL(q_\phi \parallel \pi^*) = \log Z(x) - \left( \frac{1}{\beta} \mathbb{E}_{y \sim q_\phi} \left[ r(x, y) \right] - \KL(q_\phi \parallel \pi_{\text{ref}}) \right).
\end{equation}
Since the log-partition function $\log Z(x)$ is a constant with respect to the teacher parameters $\phi$, minimizing this divergence is mathematically equivalent to maximizing the term inside the parentheses. Multiplying by $\beta$, this yields our E-step objective:
\begin{equation}
    \mathcal{J}_{\text{E-Step}}(\phi) = \mathbb{E}_{y \sim q_\phi} \left[ r(x, y) \right] - \beta \KL(q_\phi \parallel \pi_{\text{ref}}).
    \label{eq:e_step_rl}
\end{equation}
Notice that $\mathcal{J}_{\text{E-Step}}(\phi)$ takes the exact mathematical form of a standard KL-regularized RL objective (similar to Eq.~\ref{eq:rlvr}). While it is theoretically possible to optimize this via standard on-policy RL, doing so would force the teacher to independently search for successful outcomes, immediately re-introducing the severe sparse reward bottleneck we aim to bypass. Instead, we efficiently train the teacher off-policy by leveraging the diverse exploration trajectories already generated by the student.

\textbf{Off-policy Preference Optimization.}
We frame this off-policy learning as preference optimization. Since Eq.~\ref{eq:e_step_rl} mirrors the original RL objective, the closed-form optimal distribution for the teacher takes the same reward-tilted form as Eq.~\ref{eq:optimal_policy}:
$
    q^*(y \mid x, \mathcal{C}) = \frac{1}{Z(x)} \pi_{\text{ref}}(y \mid x) \exp\left( \frac{1}{\beta} r(x, y) \right).
$
By algebraically rearranging this expression and substituting our parameterized network $q_\phi$, we obtain the \emph{implicit reward} defined by the teacher's current parameters:
\begin{equation}\textstyle
    r_\phi(x, y) = \beta \log \frac{q_\phi(y \mid x, \mathcal{C})}{\pi_{\text{ref}}(y \mid x)} + \beta \log Z(x).
    \label{eq:implicit_reward}
\end{equation}
\textbf{Dynamic Reference Prior.}
In standard preference optimization, this implicit reward is anchored to a static base model. However, in a co-evolutionary framework, optimizing against a stale prior can lead to severe distribution shift between teacher and student \citep{wu2024self,rosset2024direct,pang2024iterative}. To ensure the teacher remains a useful critic for the student's current capabilities, we frame the E-step as an iterative trust-region update by dynamically anchoring the reference prior to the current student policy ($\pi_{\text{ref}} \leftarrow \pi_\theta$) \citep{schulman2015trust,schulman2017proximal}. This yields our \emph{effective implicit reward}:
\begin{equation}\textstyle
    r_\phi(x, y, \mathcal{C}) = \beta \log \frac{q_\phi(y \mid x, \mathcal{C})}{\pi_\theta(y \mid x)} + \beta \log Z(x) \triangleq \tilde{r}_\phi(x,y,\mathcal{C}) + \beta \log Z(x).
    \label{eq:effective_implicit_reward}
\end{equation}
By setting the prior to $\pi_\theta$, we redefine the optimal target $\pi^*$ as a student-relative posterior. To ensure optimization stability, we freeze the student likelihoods $\pi_\theta(y|x)$ during each E-step; we provide a rigorous analysis of this implicit reward and the resulting trust-region dynamics in Appendix~\ref{app:dynamic_prior}.

\textbf{Unpaired Preference Optimization.} Given our dynamically anchored implicit reward, if we were to optimize the teacher using the standard Bradley-Terry preference model (as in Direct Preference Optimization \citep{rafailov2023direct}), we would require paired responses $(y^+, y^-)$ evaluated under the exact same input context. However, in our framework, the diagnostic feedback $\mathcal{C}$ acts as the input context for the teacher, and this feedback is uniquely generated for each individual student trajectory $y$. Consequently, we cannot construct valid preference pairs, since there is no shared feedback context between any two distinct trajectories.

To overcome this structural bottleneck, we adopt Binary Classifier Optimization (BCO) \citep{jung2025binary}, an unpaired preference optimization framework. In a standard paired setting, the objective relies on the difference between implicit rewards, allowing the prompt-specific partition function $\beta \log Z(x)$ to perfectly cancel out. By leveraging the fundamental property of the sigmoid function, $\log \sigma(a - b) \ge \log \sigma(a) + \log \sigma(-b)$, BCO decouples the paired DPO objective into two independent parts for positive and negative samples. Substituting our computable $\tilde{r}_\phi$ terms into this inequality establishes a Binary Cross-Entropy (BCE) loss that acts as an upper bound to the standard paired DPO loss:
\begin{equation}
\begin{aligned}
    \mathcal{L}_{\text{DPO}}(\phi) &= -\mathbb{E}_{(x, y^+, y^-)} \left[ \log \sigma \left( \tilde{r}_\phi(x, y^+, \mathcal{C}^+) - \tilde{r}_\phi(x, y^-, \mathcal{C}^-) \right) \right]\\
    &\le -\mathbb{E}_{y^+} \left[ \log \sigma(\tilde{r}_\phi(x, y^+, \mathcal{C}^+)) \right] - \mathbb{E}_{y^-} \left[ \log \sigma(-\tilde{r}_\phi(x, y^-, \mathcal{C}^-)) \right].
    \label{eq:bco_bound}
\end{aligned}
\end{equation}
To minimize the approximation gap of this upper bound, we introduce a \emph{reward shift} parameter $\delta$ as prescribed by the BCO method \citep{jung2025binary}. This yields our final E-step objective:
\begin{equation}
    \mathcal{L}_{\text{E-step}}(\phi) = -\mathbb{E}_{y^+} \left[ \log \sigma(\tilde{r}_\phi(x, y^+, \mathcal{C}^+) - \delta) \right] - \mathbb{E}_{y^-} \left[ \log \sigma(-(\tilde{r}_\phi(x, y^-, \mathcal{C}^-) - \delta)) \right],
    \label{eq:e_step_final}
\end{equation}
where $\delta$ is dynamically estimated as the moving average of the batch implicit rewards: $\delta = \frac{1}{2}(\mathbb{E}[\tilde{r}_\phi(x, y^+, \mathcal{C}^+)] + \mathbb{E}[\tilde{r}_\phi(x, y^-, \mathcal{C}^-)])$.



\subsection{M-Step: Student Optimization}
\label{subsec:m_step}

With the teacher $q_\phi$ successfully refined in the E-step to approximate the local optimal policy, the M-step focuses on transferring this knowledge to the student $\pi_\theta$. Since the student operates without the privileged diagnostic feedback $\mathcal{C}$ at inference time, it must implicitly internalize the reasoning corrections discovered by the teacher. 

Mathematically, this corresponds to the maximization phase of the EM framework. Holding the teacher's parameters fixed, we project its feedback-conditioned distribution back into the student's unconditioned hypothesis space. We achieve this by minimizing the token-level KL divergence between the student and the updated teacher, sampled over the student's own on-policy rollouts:
\begin{equation}\textstyle
    \mathcal{L}_{\text{M-step}}(\theta) = \mathbb{E}_{x \sim \mathcal{D}, y \sim \pi_\theta(\cdot \mid x)} \left[ \sum_{t=1}^T \KL \left( \pi_\theta(\cdot \mid x, y_{<t}) \parallel \text{sg}\left[ q_\phi(\cdot \mid x, \mathcal{C}, y_{<t}) \right] \right) \right],
    \label{eq:m_step_final}
\end{equation}
where $\text{sg}[\cdot]$ denotes the stop-gradient operation, ensuring that optimization is strictly isolated to the student parameters $\theta$.

This projection reveals the theoretical necessity of the dynamic reference prior introduced in the E-step. Because the teacher $q_\phi$ was explicitly constrained to stay within the local trust-region of the student ($\KL(q_\phi \parallel \pi_\theta)$), we guaranteed that the teacher's target distribution remains fundamentally reachable. Consequently, this M-step distillation is highly stable, sidestepping the extreme gradient variance and mode-collapse issues that typically plague models forced to distill from a disconnected or overly dominant oracle.

As the student masters basic syntax and logic, its improved rollouts raise the baseline for the next EM cycle. The E-step then pushes the teacher to focus on increasingly complex, multi-step logical flaws, ensuring the student is continuously challenged and never starved of meaningful gradients.


\vspace{-2pt}
\subsection{Algorithm Summary}
\label{subsec:algorithm}
\vspace{-2pt}

The co-evolutionary procedure alternates between four phases: (1) gathering on-policy student rollouts, (2) generating textual critique via the environment, (3) updating the teacher via unpaired preference optimization (E-step), and (4) distilling the updated teacher into the student (M-step).

To eliminate the severe memory overhead of multi-model co-evolution, we instantiate both the student $\pi_\theta$ and teacher $q_\phi$ within a single shared-weight network ($\theta = \phi$). The distinction is purely contextual: the teacher is invoked by appending diagnostic feedback $\mathcal{C}$ to the prompt, whereas the student relies solely on the unconditioned input $x$.

This unified architecture and dynamic reference policy allow highly efficient execution on standard hardware, bypassing the need for separate frozen reference models. Furthermore, while Algorithm \ref{alg:em_vpd} depicts synchronous updates, VPD natively supports asymmetric frequencies (e.g., multiple M-steps per E-step). Updating the student more frequently acts like a target network in RL; it stabilizes the target distribution, ensuring the student internalizes guidance before the teacher advances.

Since shared-weight sequential updating shifts parameters during the E-step, initial rollouts $y \sim \pi_{\theta_{k-1}}$ become nominally off-policy for the M-step. While this could be rigorously corrected via importance sampling ($\rho = \pi_{\theta_{\text{current}}} / \pi_{\theta_{k-1}}$) as in PPO \citep{schulman2017proximal}, our E-step's constrained trust-region empirically renders this unnecessary. Omitting importance sampling yields no measurable degradation while streamlining implementation. The complete training procedure is summarized in Algorithm \ref{alg:em_vpd}.

\vspace{-3pt}
\section{Experiments}\label{sec:experiments}
\vspace{-2pt}

We design our experiments to answer three primary questions: (1) How does our EM decoupling compare to standard RL and self-distillation baselines? (2) How effectively can VPD leverage different sources of diagnostic feedback? and (3) When stress-testing on-policy self-distillation in extremely challenging scenarios, to what extent can VPD overcome these inherent limitations?

\textbf{Models and Benchmarks.} To ensure our findings generalize across different architectures and scales, our evaluation suite comprises Qwen3-1.7B, Qwen3-8B \citep{yang2025qwen3} and OLMo3-7B-Instruct \citep{olmo2025olmo}. Across our various experiments, we test reasoning capabilities spanning diverse domains: scientific reasoning via SciKnowEval (Biology, Chemistry, Material, Physics) \citep{feng2024sciknoweval}, code generation via LiveCodeBench (LCB) \citep{jain2024livecodebench}, and mathematical reasoning via DAPO-Math \citep{yu2025dapo}, Math500 \citep{lightman2023lets}, AIME24/25 \citep{aime24,aime25}, and AMC23. The specific model-benchmark pairings for each evaluation are detailed in their respective subsections below.

\textbf{Baselines.} Since our VPD framework integrates signals from both a sequence-level verifier and token-level teacher guidance, we benchmark against three classes of optimization paradigms: \textbf{(1) Pure RL:} GRPO \citep{shao2024deepseekmath}, relying entirely on sparse sequence-level rewards; \textbf{(2) Pure Distillation:} SDPO \citep{hubotter2026reinforcement}, distilling reasoning via token-level KL penalties without environment rewards; and \textbf{(3) Single-Phase Hybrids:} Three coupled mechanisms defined in Section \ref{subsec:hybrid_baselines} (Joint Loss, Advantage Reshaping, and Advantage Reweighting), which attempt to fuse RL and distillation into a simultaneous update step.

\subsection{Efficacy Across Feedback Sources}

A core component of our on-policy distillation framework is the diagnostic feedback $\mathcal{C}$. To demonstrate the versatility of VPD, we evaluate its efficacy across three distinct sources of feedback: verifiable environment execution, successful sibling rollouts, and self-critique via an LLM judge.

\begin{wraptable}{r}{0.35\linewidth}
    \centering
    \vspace{-12pt}
    \caption{Performance on the LCBv6 subset using Qwen3-8B (thinking mode off).}
    \label{tab:lcb}
    \footnotesize
    \begin{tabular}{r|c}
    \toprule
        Qwen3-8B &  28.05\\
    \midrule
        GRPO & 45.61 \\
        SDPO & 47.33 \\
        SDPO+RL (Adv Reshape) & 46.95 \\
        SDPO+RL (Adv Reweight) & 44.85 \\
        SDPO+RL (Joint Loss) & 47.52 \\
        VPD (Ours) & \textbf{49.62} \\
    \bottomrule
    \end{tabular}

    \vspace{7pt}

    \includegraphics[width=\linewidth]{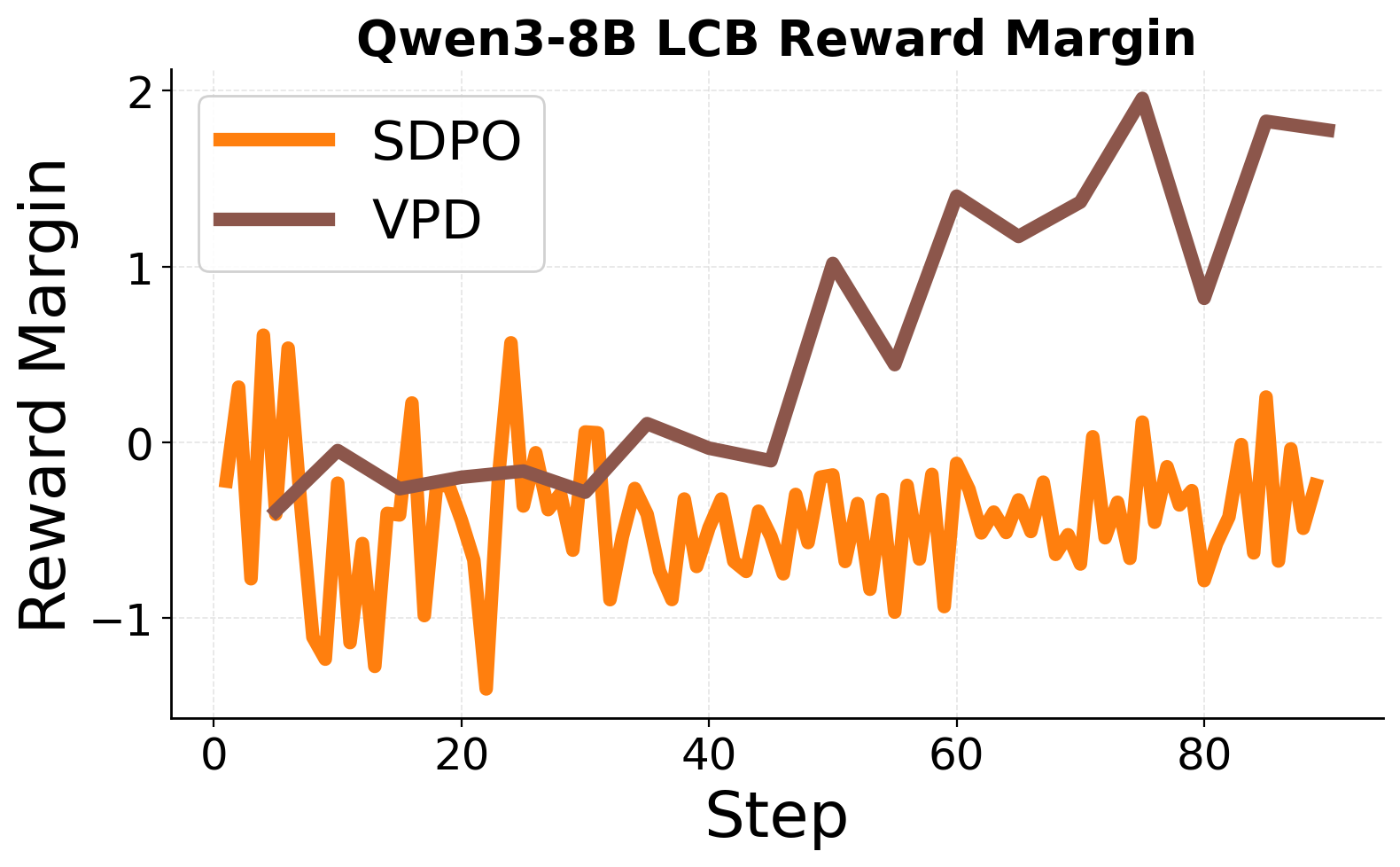}
    \vspace{-15pt}
    \captionof{figure}{Reward margin between correct and incorrect responses during LCB training. }
    \label{fig:lcb_margin}
    \vspace{-10pt}
\end{wraptable}

\textbf{1. Environment Feedback (LiveCodeBench).}
For code generation tasks, the environment acts as a natural, deterministic verifier, providing rich feedback such as runtime errors and failed unit test assertions. We evaluate Qwen3-8B (with reasoning/thinking mode disabled) on the LiveCodeBench (LCB) v6 subset, following the public and private unit test settings established by SDPO \citep{hubotter2026reinforcement}. Detailed hyperparameters and setup configurations are provided in Appendix \ref{subsec:lcb}. As shown in Table \ref{tab:lcb}, both pure RL (GRPO) and pure distillation (SDPO) significantly improve upon the base model's 28.05\% pass rate. However, attempting to simultaneously fuse these signals via single-phase hybrids yields highly mixed results: Advantage Reshaping and Advantage Reweighting degrade performance compared to pure SDPO (dropping to 46.95\% and 44.85\%, respectively), while the Joint Loss formulation provides only a marginal gain. In contrast, our decoupled VPD framework effectively internalizes the compiler's diagnostic feedback, achieving a state-of-the-art 49.62\% pass rate. To demonstrate the sustained discriminative power of our approach, Figure~\ref{fig:lcb_margin} plots the reward margin between correct and incorrect student responses during training, defined as $\tilde{r}(x,y^+,\mathcal{C}^+) - \tilde{r}(x,y^-,\mathcal{C}^-)$. As training progresses, we observe that the margin for standard SDPO rapidly diminishes, indicating a loss of useful guidance. Conversely, VPD consistently increases this margin, confirming that our E-step successfully continuously refines the teacher to distinguish between high- and low-quality trajectories.

\begin{wraptable}{r}{0.56\linewidth}
\centering
\vspace{-2pt}
\caption{Performance on the SciKnowEval benchmark using contrastive sibling rollouts as feedback.}
\label{tab:sciqa}
\footnotesize
\setlength{\tabcolsep}{3pt}
\begin{tabular}{rccccc}
\toprule
 & \textbf{Bio.} & \textbf{Chem.} & \textbf{Mat.} & \textbf{Phys.} & \textbf{AVG} \\
\midrule
\rowcolor{gray!20}
\textbf{Qwen3-1.7B} & 33.63 & 40.95 & 50.07 & 49.92 & 43.64 \\
GRPO & 61.25 & 75.65 & 74.93 & 67.42 & 69.81 \\
SDPO & 61.50 & 72.41 & 71.54 & 59.92 & 66.34 \\
SDPO+RL (Adv Reshape) & 54.62 & 54.94 & 66.62 & 58.83 & 58.75 \\
SDPO+RL (Adv Reweight) & 57.75 & 73.36 & 75.27 & 62.58 & 67.24 \\
SDPO+RL (Joint Loss) & 58.00 & 77.53 & 68.88 & 62.11 & 66.63 \\
VPD (Ours) & \textbf{64.75} & \textbf{81.88} & \textbf{77.06} & \textbf{73.67} & \textbf{74.34} \\
\midrule
\rowcolor{gray!20}
\textbf{QWEN3-8B} & 31.13 & 42.14 & 59.18 & 59.14 & 47.90 \\
GRPO & 62.50 & 76.58 & 77.26 & 76.09 & 73.11 \\
SDPO & 61.62 & 82.26 & \textbf{79.59} & 74.30 & 74.44 \\
SDPO+RL (Adv Reshape) & 52.00 & 75.18 & 75.07 & 74.14 & 69.10 \\
SDPO+RL (Adv Reweight) & 61.37 & 76.46 & 75.80 & 70.39 & 71.01 \\
SDPO+RL (Joint Loss) & 66.00 & 81.01 & 73.87 & 73.12 & 73.50 \\
VPD (Ours) & \textbf{68.00} & \textbf{82.38} & 77.66 & \textbf{80.55} & \textbf{77.15} \\
\midrule
\rowcolor{gray!20}
\textbf{Olmo3-7B-Instruct} & 16.12 & 22.77 & 34.71 & 37.34 & 27.74 \\
GRPO & 52.88 & 69.88 & 74.00 & 66.09 & 65.71 \\
SDPO & 50.00 & 80.15 & 70.21 & 63.91 & 66.07 \\
SDPO+RL (Adv Reshape) & 50.25 & 65.83 & 72.47 & 66.02 & 63.64 \\
SDPO+RL (Adv Reweight) & \textbf{58.25} & 76.67 & 74.87 & 66.56 & 69.09 \\
SDPO+RL (Joint Loss) & 55.37 & \textbf{80.51} & 71.21 & 69.45 & 69.14 \\
VPD (Ours) & 55.62 & 80.33 & \textbf{76.06} & \textbf{71.17} & \textbf{70.80} \\
\bottomrule
\end{tabular}

\vspace{10pt} 

\includegraphics[width=0.48\linewidth]{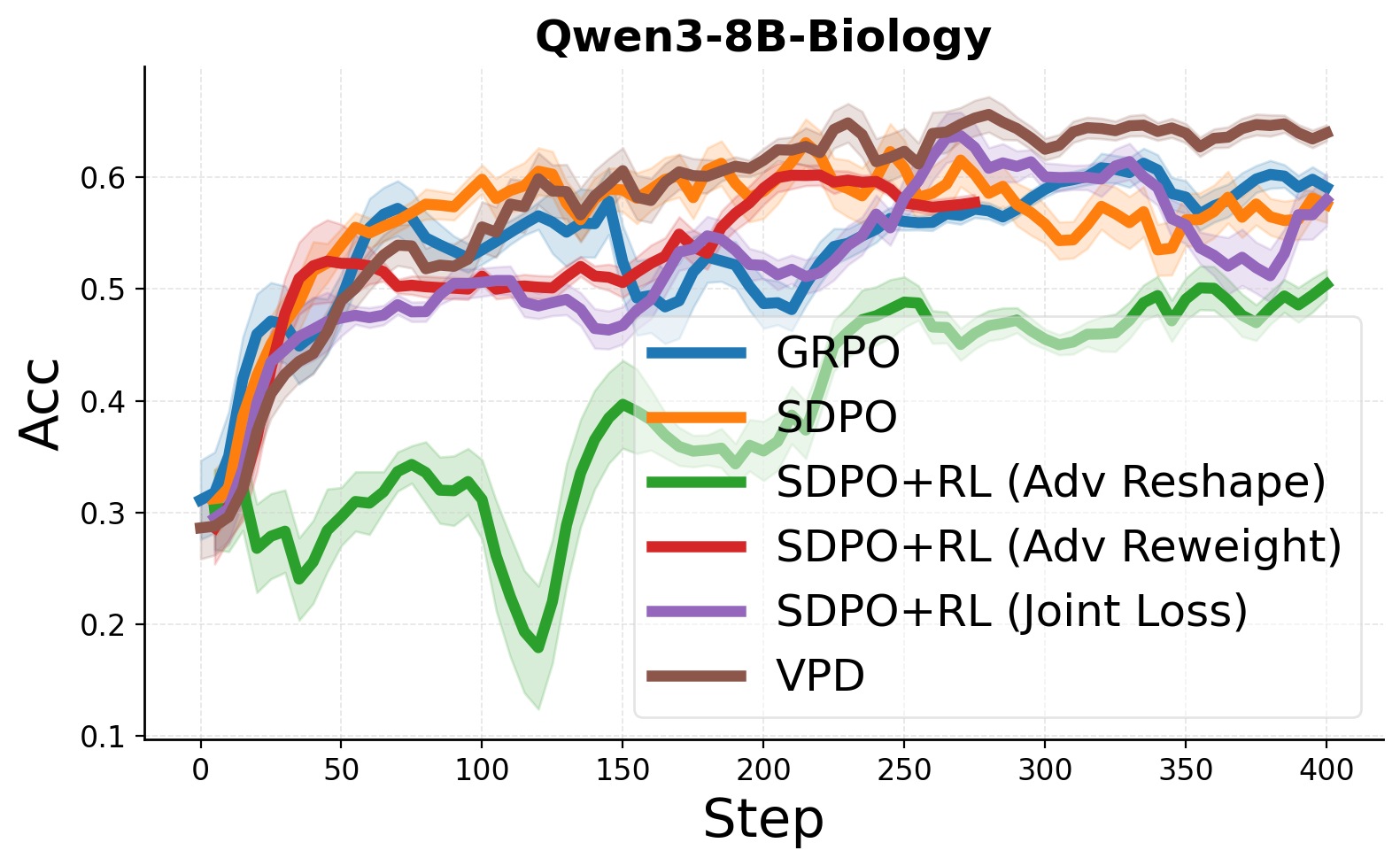}
\hfill
\includegraphics[width=0.48\linewidth]{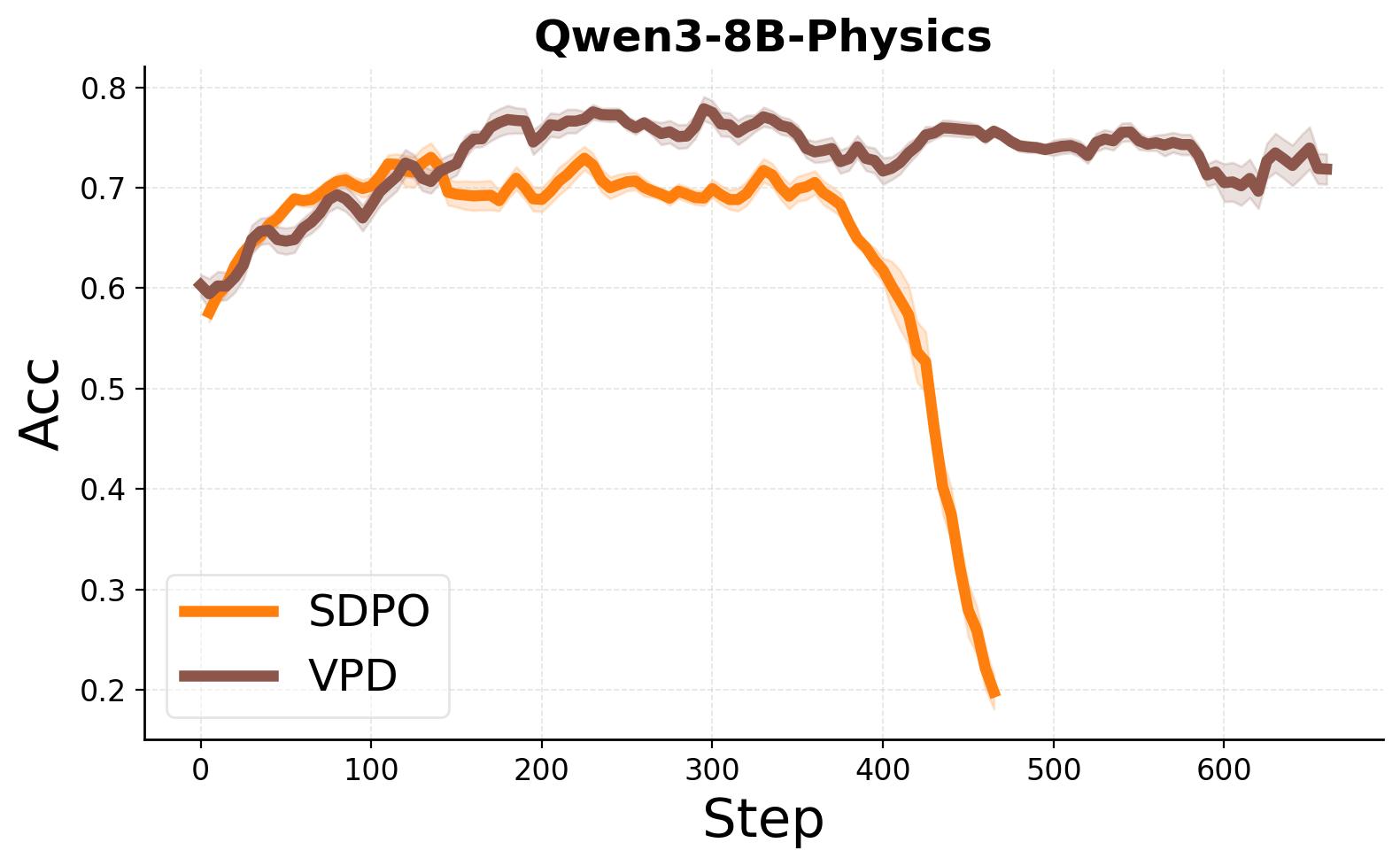}
\captionof{figure}{Training progression on SciKnowEval}
\label{fig:sciqa}

\end{wraptable}

\textbf{2. Contrastive Sibling Rollouts (SciKnowEval).}
For many scientific reasoning tasks, ground-truth textual feedback is unavailable; the environment only provides a sparse, binary correctness signal. In these scenarios, we can synthesize the diagnostic feedback $\mathcal{C}$ using the model's own generations. Following the methodology of SDPO \citep{hubotter2026reinforcement}, we provide the student with a successful trajectory sampled from the same prompt's rollout group, effectively using a contrastive sibling rollout as the textual guidance. We evaluate this setting on the SciKnowEval benchmark across Qwen3-1.7B, Qwen3-8B, and OLMo3-7B-Instruct (all with thinking mode disabled). As detailed in Table~\ref{tab:sciqa}, VPD consistently outperforms both pure baselines and single-phase hybrids across almost all domains. Looking at the aggregate averages, VPD establishes a clear state-of-the-art: achieving 74.34\% on Qwen3-1.7B (vs. GRPO's 69.81\%), 77.15\% on Qwen3-8B (vs. SDPO's 74.44\%), and 70.80\% on OLMo3-7B-Instruct. On specific sub-tasks, the gains are even more pronounced; for instance, on Qwen3-8B (Biology), VPD achieves 68.00\% compared to GRPO's 62.50\% and SDPO's 61.62\%. Crucially, this performance advantage is intrinsically linked to training stability. Across the SciKnowEval sub-tasks, the single-phase hybrid baselines (SDPO+RL variants) experience severe training instability. We hypothesize that simultaneously updating a policy using unbounded log-ratios and high-variance RL advantages causes catastrophic scale mismatches—a vulnerability entirely bypassed by VPD's decoupled EM formulation. Furthermore, as illustrated in Figure \ref{fig:sciqa}, standard SDPO frequently suffers from late-stage training degradation, where validation accuracy begins to decrease at longer step counts. In contrast, VPD eliminates this issue entirely, exhibiting a highly stable and monotonic convergence curve.

\begin{wraptable}{r}{0.5\linewidth}
\centering
\vspace{-10pt}
\caption{Performance on SciKnowEval using autonomous self-critique as the feedback source.}
\label{tab:judge}
\footnotesize
\setlength{\tabcolsep}{3pt}
\begin{tabular}{llccccc}
\toprule
& & \textbf{Bio.} & \textbf{Chem.} & \textbf{Mat.} & \textbf{Phys.} & \textbf{AVG} \\
\midrule
\multirow{2}{*}{\textbf{Qwen3-1.7B}} 
& SDPO & 62.12 & 78.66 & 71.21 & 58.13 & 67.53 \\
& VPD & 64.75 & 79.43 & 73.54 & 70.31 & 72.01 \\
\midrule
\multirow{2}{*}{\textbf{QWEN3-8B}} 
& SDPO & 61.38 & 82.35 & 83.71 & 72.03 & 74.87 \\
& VPD & 65.38 & 83.52 & 84.20 & 79.45 & 78.14 \\
\bottomrule
\end{tabular}
\end{wraptable}

\textbf{3. Self-Critique via LLM Judge.}
Contrastive sibling rollouts fail when all sampled trajectories for a prompt are incorrect, providing no diagnostic feedback. To overcome this, we evaluate a setting where the model acts as its own judge, generating autonomous critiques for failed trajectories without positive pairs (setup in Appendix \ref{subsec:judge}). This mimics a self-reflection loop: the model diagnoses its own logical errors, and the E-step refines the teacher policy based on this reflection before distilling it back into the student. As Table \ref{tab:judge} shows, self-critique yields performance comparable to contrastive siblings. Crucially, VPD still significantly outperforms SDPO across architectures (e.g., achieving 78.14\% vs. 74.87\% on Qwen3-8B). This viability highlights a compelling future direction: dynamically co-evolving the judge alongside the reasoning policy.

\vspace{-5pt}
\subsection{Ablations and Analysis: The Bounds of Self-Distillation}
\vspace{-2pt}

While VPD successfully stabilizes hybrid training across science and coding domains, our empirical investigation revealed two critical scenarios where all self-distillation methods (including standard SDPO and our VPD) struggle relative to pure RL. We present these findings to illuminate the fundamental, empirical limits of language-driven self-distillation.

\begin{wrapfigure}{r}{0.3\linewidth}
    \centering
    \vspace{-10pt}
    \includegraphics[width=\linewidth]{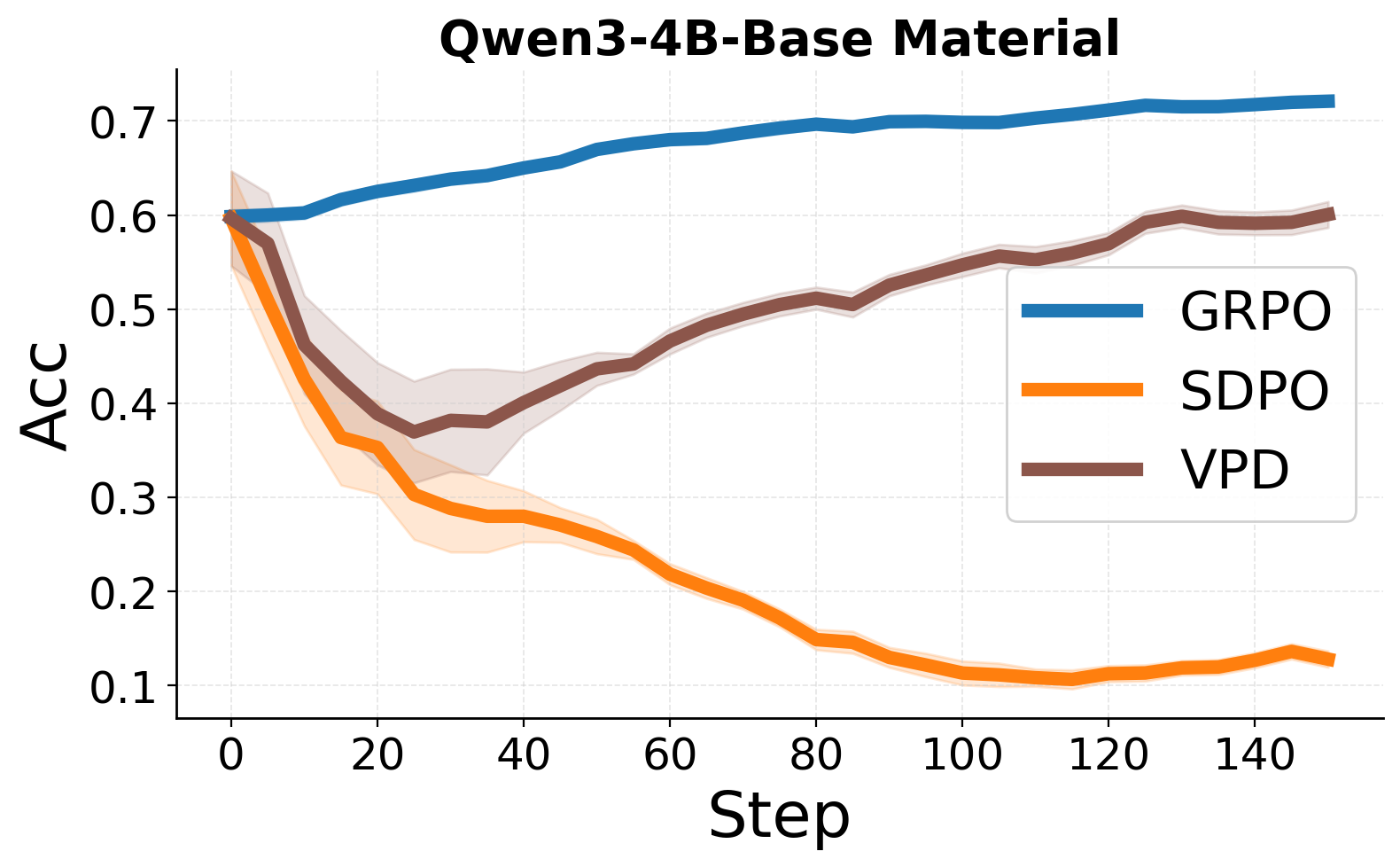}
    \caption{Training progression on Qwen3-4B-Base.}
    \label{fig:base_model}
\end{wrapfigure}

\textbf{The "Cold Start" Problem on Base Models.} 
Recent literature demonstrates that GRPO can elicit advanced reasoning capabilities from a base foundation model. However, when we apply SDPO to base models, performance rapidly collapses to near zero. We hypothesize that self-distillation intrinsically requires the policy to possess a rudimentary level of instruction-following competence; if the base model lacks the capacity to properly digest the diagnostic feedback $\mathcal{C}$ in its prompt, the teacher's target distribution becomes corrupted. While our VPD formulation mitigates this collapse---increasing discrimination power via the E-step updates and significantly delaying the degradation---it still ultimately underperforms GRPO in pure cold-start scenarios. Figure \ref{fig:base_model} illustrates this delayed collapse, confirming that while VPD is far more robust than SDPO, pure RL remains necessary for emergent reasoning elicitation.

\begin{wrapfigure}{r}{0.3\linewidth}
    \centering
    \vspace{-10pt}
    \includegraphics[width=\linewidth]{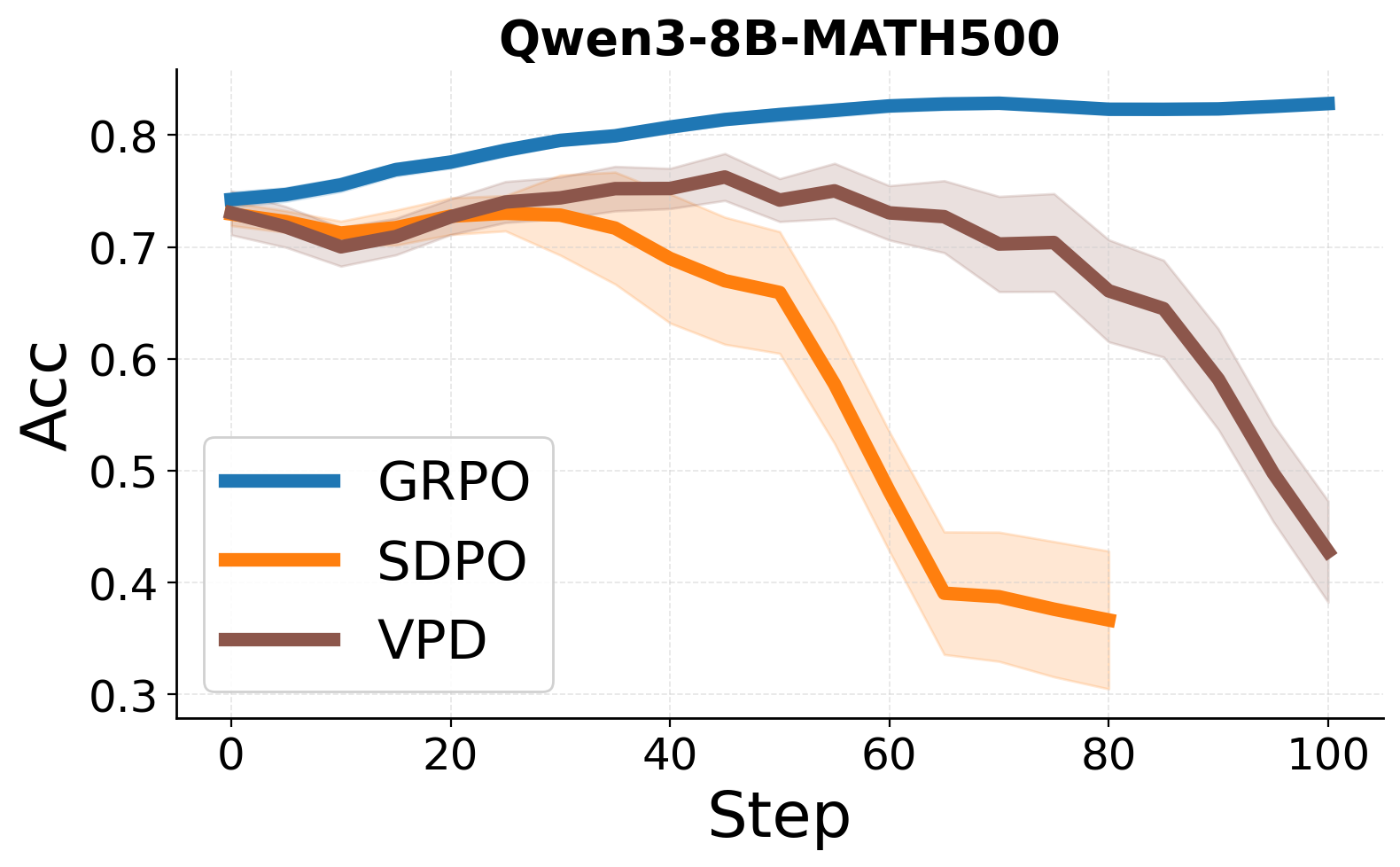}
    \caption{Performance on the Math500 benchmark for models trained on DAPO-Math.}
    \label{fig:math}
    \vspace{-8pt}
\end{wrapfigure}

\textbf{Mathematical Reasoning.} 
Similarly, on challenging mathematical benchmarks (e.g., training on DAPO-Math), SDPO suffers from severe training collapse. This vulnerability to mathematical reasoning domains has been observed in concurrent works \citep{kim2026does}. While VPD again successfully delays this collapse, pure GRPO remains the dominant approach (achieving 83.8\% on Math500 with Qwen3-8B). We hypothesize that this stems from the strict, non-forgiving nature of mathematical derivations. Sparse RL (GRPO) encourages broad exploration, rewarding the model only when it independently discovers a rigorous, fully correct logical path. In contrast, self-distillation forces the student to closely track the teacher's intermediate token distribution. If the teacher's diagnostic feedback is imprecise or flawed, distilling this noisy guidance may overly constrain exploration and inadvertently reinforce incorrect reasoning steps.

\textbf{Computational Efficiency.}
A major practical advantage of VPD is its memory and sampling efficiency. By instantiating both policies within a single shared-weight network, we completely eliminate the massive VRAM overhead of hosting separate teacher models. Furthermore, because the exact same on-policy student rollouts are shared across both the E-step and M-step, VPD incurs zero additional sampling or environment verification costs. While the explicit E-step introduces a gradient-computation overhead compared to standard SDPO—empirically observed as a 30\% to 55\% increase in overall runtime—this temporal cost is effectively managed by our asymmetric update frequency (e.g., one E-step per five M-steps). Ultimately, the significant gains in training stability and final reasoning performance strongly justify this moderate runtime increase, establishing VPD as a highly practical framework.

\begin{table}[ht]
    \centering
    \begin{minipage}[t]{0.51\linewidth}
        \centering
        \caption{Ablation on E-step update frequency using Qwen3-1.7B. $N$ denotes one E-step per $N$ M-steps.}
        \label{tab:freq}
        \footnotesize
        \setlength{\tabcolsep}{3pt}
        \begin{tabular}{lccccc}
        \toprule
         & \textbf{Bio.} & \textbf{Chem.} & \textbf{Mat.} & \textbf{Phys.} & \textbf{AVG} \\
        \midrule
        VPD (F5) & 64.75 & 81.88 & 77.06 & 73.67 & 74.34 \\
        VPD (F1) & 61.12 & 78.51 & 72.47 & 68.75 & 70.21 \\
        VPD (F10) & 59.50 & 79.50 & 71.61 & 66.48 & 69.27 \\
        \bottomrule
        \end{tabular}
    \end{minipage}
    \hfill
    \begin{minipage}[t]{0.46\linewidth}
        \centering
        \caption{Ablation on the E-Step reference prior using Qwen3-1.7B on SciKnowEval.}
        \label{tab:prior}
        \footnotesize
        \setlength{\tabcolsep}{3pt}
        \begin{tabular}{lccccc}
        \toprule
        \textbf{Prior Type} & \textbf{Bio.} & \textbf{Chem.} & \textbf{Mat.} & \textbf{Phys.} & \textbf{AVG} \\
        \midrule
            Dynamic ($\pi_\theta$) & 64.75 & 81.88 & 77.06 & 73.67 & 74.34 \\
            Fixed ($\pi_{\text{ref}}$) & 55.87 & 76.34 & 72.81 & 66.33 & 67.84 \\
        \bottomrule
        \end{tabular}
    \end{minipage}
\end{table}

\textbf{Ablation: E-Step Update Frequency.} 
As discussed in Sec.~\ref{subsec:algorithm}, our shared-weight architecture allows for asymmetric update frequencies. In our main experiments, we perform one E-step update for every 5 M-step gradient updates (F5). To validate this design, we ablate the update frequency on the SciKnowEval benchmark using Qwen3-1.7B. As shown in Table~\ref{tab:freq}, both overly frequent and overly infrequent E-step updates degrade performance. If the teacher is updated too frequently (F1), the target distribution becomes volatile, functioning like a rapidly moving target network in RL that destabilizes the student's distillation phase. Conversely, if the teacher is updated too infrequently (F10), the target distribution becomes stale, preventing the student from receiving dynamically adjusted feedback. The F5 configuration provides the optimal balance, allowing the student sufficient gradient steps to internalize the refined guidance before the teacher advances.

\begin{wrapfigure}{r}{0.3\linewidth}
    \centering
    \includegraphics[width=\linewidth]{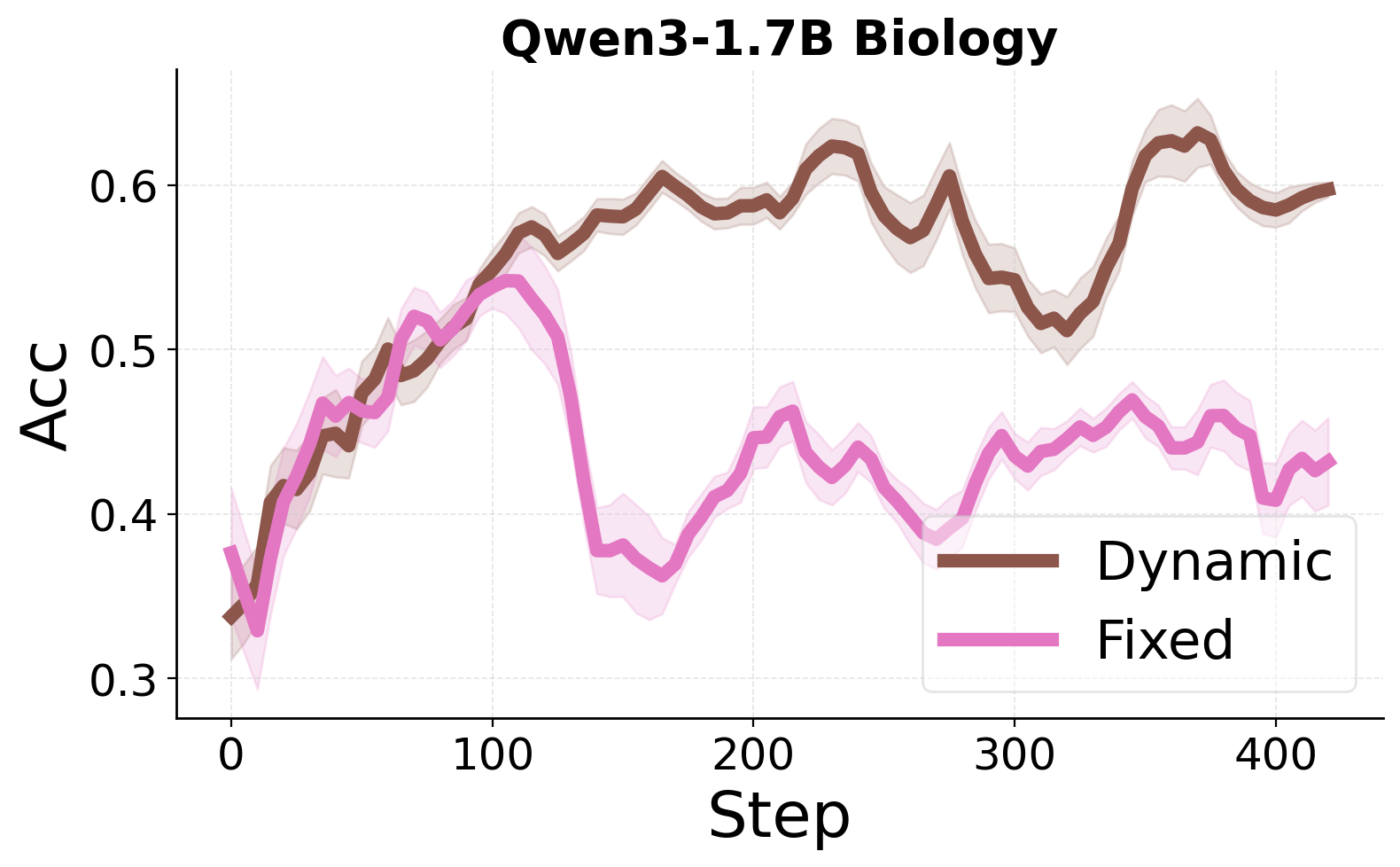}
    \caption{Training progression on Qwen3-1.7B with different reference model for E-Step.}
    \label{fig:prior}
\end{wrapfigure}

\textbf{Ablation: Dynamic Reference Prior.}
As established in Eq.~\ref{eq:effective_implicit_reward}, VPD dynamically anchors the reference prior to the current student policy ($\pi_\theta$). This sliding trust region restricts the teacher's target distribution, ensuring its guidance remains safely reachable for the student. To validate this design, we conduct an ablation study comparing our dynamic prior against a fixed reference model ($\pi_{\text{ref}}$). As shown in Table~\ref{tab:prior}, reverting to a static prior on Qwen3-1.7B not only results in a severe performance drop across the evaluation domains but also severely degrades training stability (see Fig.~\ref{fig:prior}). This confirms our hypothesis: optimizing against a stale prior induces harmful distribution shift, decoupling the teacher's guidance from the student's exploration space and destabilizing the M-step distillation.

\vspace{-5pt}
\section{Related Work}\label{sec:related_work}
\vspace{-6pt}

\textbf{Reinforcement Learning with Verifiable Rewards (RLVR).} 
Recent advancements in large foundation models have popularized Reinforcement Learning with Verifiable Rewards (RLVR) \citep{guo2025deepseek,shao2024deepseekmath,yang2025qwen3}. Unlike traditional RLHF which relies on noisy human preference models \citep{ouyang2022training}, RLVR optimizes LLMs against deterministic, rule-based environments (e.g., Python interpreters, formal theorem provers, or exact-match graders). While RLVR provides unbiased ground-truth signals, it suffers heavily from reward sparsity; a model receives a binary outcome only at the end of a long generative trajectory, making the credit assignment problem exceptionally difficult \citep{kazemnejad2024vineppo,zhang2024generative,zhang2025improving,zheng2025act,liu2025inference}. Our Variational Policy Distillation (VPD) framework fundamentally solves this sparsity issue. Instead of treating the verifiable environment ($\mathcal{E}$) as a black-box scalar reward generator, we extract dense diagnostic feedback ($\mathcal{C}$) from it to condition an intermediate teacher, transforming sparse RLVR into a rich, token-level distillation process.

\textbf{Iterative Self-Improvement and Preference Optimization.} 
A rapidly growing body of literature explores allowing LLMs to iteratively self-improve using their own generated rollouts. Frameworks like Self-Play Fine-Tuning (SPIN) \citep{chen2024self} pit a model against its previous iterations, while Self-Rewarding LLMs \citep{yuan2024self} prompt the model to act as its own reward judge. To bypass the instabilities of policy gradients (e.g., PPO), recent methods iteratively apply Direct Preference Optimization (DPO) \citep{rafailov2023direct} or leverage unpaired preference learning like Kahneman-Tversky Optimization (KTO) \citep{ethayarajh2024kto} and Binary Classifier Optimization (BCO) \citep{jung2025binary}. However, unconstrained self-improvement risks severe distribution shift and mode collapse, as the target distribution can easily decouple from the student's active exploration space.

\textbf{Privileged Information and On-Policy Self-Distillation.} 
To mitigate distribution shift and leverage privileged information without relying on external oracles, recent research has shifted toward On-Policy Self-Distillation. Methods such as On-Policy Self-Distillation (OPSD) \citep{zhao2026self} condition a teacher on ground-truth answers to generate distillation targets. Similarly, On-Policy Context Distillation (OPCD) \citep{ye2026policy,ye2026online} distills historical solution traces and system-prompt behaviors into model weights, while Self-Distillation Policy Optimization (SDPO) \citep{hubotter2026reinforcement} conditions a self-teacher on diagnostic feedback for dense credit assignment. A major limitation of these existing methods is that they typically treat the conditionally prompted teacher as a fixed, heuristic oracle during the update step. Conceptually closer to our work is $\pi$-Distill \citep{penaloza2026privileged}, which moves beyond a static oracle by jointly learning a teacher (via GRPO) and a student. However, it treats the teacher's rollouts as disconnected off-policy samples to subsequently fine-tune the student. 

VPD fundamentally departs from all these approaches by embedding the self-distillation process within a rigorous Expectation-Maximization (EM) framework. Rather than blindly distilling from a static prompt or relying on ad-hoc off-policy sampling, our E-step actively \emph{optimizes} the teacher to sharply distinguish successful from failed trajectories. Crucially, this optimization is mathematically constrained within a dynamic trust region. This guarantees that the teacher's target distribution is not only highly informative but remains safely anchored to the student's local support, ensuring stable M-step distillation and preventing the vanishing gradients that plague standard self-distillation. More broadly, while concurrent work explores variational inference for LLM reasoning \citep{zhou2025variational}, VPD uniquely formulates this latent-variable optimization as a teacher-student distillation problem.

\textbf{Learning from Language Feedback.} 
Moving beyond sparse scalar rewards, a rapidly expanding body of work explores natural language as a denser, more informative supervision signal. Theoretical foundations have recently been laid by tracking LLM hypothesis spaces and utilizing the transfer eluder dimension to quantify feedback information \citep{xu2025provably}, as well as by leveraging natural language for explicit value estimation and policy improvement \citep{feng2024natural}. In practice, methods integrate feedback through various mechanisms: some utilize verbal reflections for temporal credit assignment \citep{chen2026retrospective,zhong2024policy,yang2025deepcritic}, while others treat feedback as textual gradients to iteratively evolve responses \citep{lee2025feedback,song2026expanding}. Another prominent strategy conditions models directly on the feedback itself during fine-tuning, as seen in Chain of Hindsight \citep{liu2023chain}, Imitation learning from Language Feedback (ILF) \citep{scheurer2023training}, and Feedback-Conditional Policy (FCP) \citep{luo2025language}.

While these approaches successfully harness rich linguistic signals, they typically require training separate, computationally expensive critique models or rely on unconstrained, heuristic fine-tuning. VPD streamlines this paradigm. By fusing the feedback-aware teacher and the student into a single, shared-weight network ($\theta = \phi$), we embed dense language critiques ($\mathcal{C}$) directly into a mathematically bounded Expectation-Maximization cycle, achieving the benefits of language feedback without the traditional memory overhead or instability.

\section{Conclusion}\label{sec:conclusion}

In this work, we introduced Variational Policy Distillation (VPD), a framework that reframes self-distillation as a principled, co-evolutionary Expectation-Maximization (EM) algorithm. Rather than relying on frozen heuristic oracles or static reference models, VPD jointly trains a feedback-aware teacher alongside an unconditioned student. During the E-step, the framework actively updates the teacher via unpaired preference learning (BCO) to extract sharper insights from environmental critiques ($\mathcal{C}$). Crucially, by anchoring this evolving teacher to the active student policy, we enforce a dynamic trust-region that ensures the subsequent M-step distillation remains highly informative yet safely within the student's learning capacity. We execute this complex cycle highly efficiently through a unified, shared-weight architecture ($\phi = \theta$), completely eliminating the massive memory overhead of traditional multi-model paradigms. Our theoretical formulation and empirical evaluations demonstrate that this approach yields stable optimization dynamics and consistent performance gains on complex reasoning tasks, outperforming static reference methods like standard SDPO. 

\textbf{Limitations and Future Work.} Despite these advancements, our framework has notable limitations. First, while VPD significantly improves upon standard self-distillation and successfully delays training collapse in challenging scenarios, it still ultimately lags behind pure sparse RL (e.g., GRPO) in mathematical reasoning and base-model cold-starts. This highlights an inherent limitation in relying on unstructured textual feedback, which can sometimes be too noisy or imprecise for strict logical derivations. Second, while our shared-weight parameterization ($\phi = \theta$) maximizes computational efficiency, it strictly bounds the teacher's representational capacity to the student's architecture. In highly complex domains, this strict coupling may prevent the teacher from sufficiently approximating the optimal posterior. Moving forward, future work should explore decoupled or partially decoupled parameterizations (e.g., using parameter-efficient fine-tuning for the teacher) to expand its diagnostic capacity without reintroducing massive memory overheads. Additionally, analyzing the theoretical bounds of the EM framework under varying degrees of feedback noise, and extending this co-evolutionary process to continual learning scenarios, represent exciting frontiers for developing autonomous, self-improving agents.

\bibliographystyle{plain}
\bibliography{neurips_2026}


\newpage
\appendix

\setcounter{equation}{0}
\setcounter{figure}{0}
\setcounter{table}{0}
\renewcommand{\theequation}{\thesection.\arabic{equation}}
\renewcommand{\thefigure}{\thesection.\arabic{figure}}
\renewcommand{\thetable}{\thesection.\arabic{table}}

\section{Theoretical Derivations}\label{sec:theory_app}

This appendix provides the formal derivations for the variational framework introduced in Section \ref{sec:method}. We first derive the closed-form optimal policy under the KL-regularized RLVR objective, demonstrate its equivalence to minimizing the reverse KL divergence, and finally establish the Expectation-Maximization (EM) lower bound that justifies the alternating optimization of the teacher $q_\phi$ and student $\pi_\theta$.

\subsection{Derivation of the Optimal Target Distribution}\label{app:optimal_policy}

The standard KL-regularized RLVR objective for a given context $x$ is defined as:
\begin{equation}
    \mathcal{J}(\pi) = \mathbb{E}_{y \sim \pi(\cdot|x)} \left[ r(x, y) \right] - \beta \KL \left( \pi(\cdot \mid x) \parallel \pi_{\text{ref}}(\cdot \mid x) \right).
\end{equation}

We wish to find the optimal policy $\pi^*$ that maximizes $\mathcal{J}(\pi)$ subject to the probability constraint $\sum_y \pi(y \mid x) = 1$. We formulate this constrained optimization problem using the method of Lagrange multipliers. The Lagrangian is given by:
\begin{equation}
    \mathcal{L}(\pi, \lambda) = \sum_y \pi(y \mid x) r(x, y) - \beta \sum_y \pi(y \mid x) \log \frac{\pi(y \mid x)}{\pi_{\text{ref}}(y \mid x)} - \lambda \left( \sum_y \pi(y \mid x) - 1 \right).
\end{equation}

Taking the partial derivative of $\mathcal{L}$ with respect to $\pi(y \mid x)$ and setting it to zero yields:
\begin{equation}
    \frac{\partial \mathcal{L}}{\partial \pi(y \mid x)} = r(x, y) - \beta \left( \log \frac{\pi(y \mid x)}{\pi_{\text{ref}}(y \mid x)} + 1 \right) - \lambda = 0.
\end{equation}

Solving for $\pi(y \mid x)$:
\begin{equation}
    \log \frac{\pi(y \mid x)}{\pi_{\text{ref}}(y \mid x)} = \frac{1}{\beta} r(x, y) - 1 - \frac{\lambda}{\beta} \implies \pi^*(y \mid x) = \pi_{\text{ref}}(y \mid x) \exp \left( \frac{1}{\beta} r(x, y) \right) \exp \left( -1 - \frac{\lambda}{\beta} \right).
\end{equation}

To eliminate the Lagrange multiplier $\lambda$, we apply the constraint that the probabilities must sum to 1. This reveals that the term $\exp(-1 - \frac{\lambda}{\beta})$ acts as a normalization constant. We define the partition function $Z(x)$ as:
\begin{equation}
    Z(x) = \sum_y \pi_{\text{ref}}(y \mid x) \exp \left( \frac{1}{\beta} r(x, y) \right).
\end{equation}

Thus, the optimal target distribution is the exponentially reward-tilted policy:
\begin{equation}
    \pi^*(y \mid x) = \frac{1}{Z(x)} \pi_{\text{ref}}(y \mid x) \exp \left( \frac{1}{\beta} r(x, y) \right).
    \label{eq:app_pi_star}
\end{equation}

\subsection{Equivalence of Reverse KL and the RLVR Objective}\label{app:kl_equivalence}

We now demonstrate that minimizing the reverse KL divergence between the student policy $\pi_\theta$ and the optimal policy $\pi^*$ is mathematically equivalent to maximizing the original RLVR objective. Expanding the definition of the KL divergence:
\begin{equation}
    \KL(\pi_\theta \parallel \pi^*) = \mathbb{E}_{y \sim \pi_\theta} \left[ \log \pi_\theta(y \mid x) - \log \pi^*(y \mid x) \right].
    \label{eq:optimal_kl_obj}
\end{equation}

Substituting the definition of $\pi^*(y \mid x)$ from Equation \ref{eq:app_pi_star}:
\begin{align}
    \KL(\pi_\theta \parallel \pi^*) &= \mathbb{E}_{y \sim \pi_\theta} \left[ \log \pi_\theta(y \mid x) - \log \left( \frac{1}{Z(x)} \pi_{\text{ref}}(y \mid x) \exp \left( \frac{1}{\beta} r(x, y) \right) \right) \right] \nonumber \\
    &= \mathbb{E}_{y \sim \pi_\theta} \left[ \log \pi_\theta(y \mid x) - \log \pi_{\text{ref}}(y \mid x) - \frac{1}{\beta} r(x, y) + \log Z(x) \right] \nonumber \\
    &= \KL(\pi_\theta \parallel \pi_{\text{ref}}) - \frac{1}{\beta} \mathbb{E}_{y \sim \pi_\theta} \left[ r(x, y) \right] + \log Z(x).
\end{align}

Multiplying both sides by $-\beta$, we observe:
\begin{equation}
    -\beta \KL(\pi_\theta \parallel \pi^*) = \mathbb{E}_{y \sim \pi_\theta} \left[ r(x, y) \right] - \beta \KL(\pi_\theta \parallel \pi_{\text{ref}}) - \beta \log Z(x).
    \label{eq:kl_to_rlvr}
\end{equation}

Since the term $\beta \log Z(x)$ depends only on the environment and $\pi_{\text{ref}}$, it is a constant with respect to the policy parameters $\theta$. Therefore, minimizing the reverse KL divergence exactly maximizes the original objective $\mathcal{J}(\theta)$.

\subsection{Variational Expectation-Maximization Lower Bound}
\label{app:em_bound}

Because $Z(x)$ is analytically intractable, computing the exact log-likelihood of $\pi^*$ is impossible. We introduce a tractable teacher distribution $q_\phi(y \mid x, \mathcal{C})$ conditioned on the dense language feedback $\mathcal{C}$ to act as an approximate posterior. 

To evaluate the quality of this approximation, we look to the log-partition function, $\log Z(x)$. By rearranging the equivalence established in Equation \ref{eq:kl_to_rlvr}, we can express the original RL objective as:
\begin{equation}\label{eq:log_part_bound}
    \mathcal{J}(\pi) = \beta \log Z(x) - \beta \KL(\pi \parallel \pi^*).
\end{equation}
Because the KL divergence is strictly non-negative, the maximum possible value of the objective function is exactly $\beta \log Z(x)$, achieved only when $\pi$ perfectly matches $\pi^*$. Therefore, $\log Z(x)$ represents the absolute theoretical ceiling of performance under the KL constraint.

To bound this intractable ceiling, we multiply and divide the terms inside $Z(x)$ by our teacher distribution $q_\phi(y \mid x, \mathcal{C})$:
\begin{equation}
    \log Z(x) = \log \sum_y q_\phi(y \mid x, \mathcal{C}) \frac{\pi_{\text{ref}}(y \mid x) \exp \left( \frac{1}{\beta} r(x, y) \right)}{q_\phi(y \mid x, \mathcal{C})}.
\end{equation}

Applying Jensen's Inequality, we establish the Evidence Lower Bound (ELBO), denoted as $\mathcal{F}(q_\phi)$:
\begin{align}
    \log Z(x) &\ge \sum_y q_\phi(y \mid x, \mathcal{C}) \log \frac{\pi_{\text{ref}}(y \mid x) \exp \left( \frac{1}{\beta} r(x, y) \right)}{q_\phi(y \mid x, \mathcal{C})} \nonumber \\
    &= \frac{1}{\beta} \mathbb{E}_{y \sim q_\phi} \left[ r(x, y) \right] - \KL(q_\phi \parallel \pi_{\text{ref}}) \equiv \mathcal{F}(q_\phi).
\end{align}

To explicitly derive the gap between this lower bound and the true optimal log-partition function, we expand the KL divergence between our approximate posterior $q_\phi$ and the optimal policy $\pi^*$:
\begin{equation}
    \KL(q_\phi \parallel \pi^*) = \mathbb{E}_{y \sim q_\phi} \left[ \log q_\phi(y \mid x, \mathcal{C}) - \log \pi^*(y \mid x) \right].
\end{equation}

Substituting the definition of $\pi^*(y \mid x)$ from Equation \ref{eq:app_pi_star}:
\begin{align}
    \KL(q_\phi \parallel \pi^*) &= \mathbb{E}_{y \sim q_\phi} \left[ \log q_\phi(y \mid x, \mathcal{C}) - \log \left( \frac{1}{Z(x)} \pi_{\text{ref}}(y \mid x) \exp \left( \frac{1}{\beta} r(x, y) \right) \right) \right] \nonumber \\
    &= \mathbb{E}_{y \sim q_\phi} \left[ \log q_\phi(y \mid x, \mathcal{C}) - \log \pi_{\text{ref}}(y \mid x) - \frac{1}{\beta} r(x, y) + \log Z(x) \right].
\end{align}

By separating the terms, we can reformulate this in terms of the ELBO:
\begin{align}
    \KL(q_\phi \parallel \pi^*) &= \log Z(x) + \mathbb{E}_{y \sim q_\phi} \left[ \log \frac{q_\phi(y \mid x, \mathcal{C})}{\pi_{\text{ref}}(y \mid x)} \right] - \frac{1}{\beta} \mathbb{E}_{y \sim q_\phi} \left[ r(x, y) \right] \nonumber \\
    &= \log Z(x) + \KL(q_\phi \parallel \pi_{\text{ref}}) - \frac{1}{\beta} \mathbb{E}_{y \sim q_\phi} \left[ r(x, y) \right] \nonumber \\
    &= \log Z(x) - \left( \frac{1}{\beta} \mathbb{E}_{y \sim q_\phi} \left[ r(x, y) \right] - \KL(q_\phi \parallel \pi_{\text{ref}}) \right).
\end{align}

Recognizing the term inside the parentheses as exactly our definition of $\mathcal{F}(q_\phi)$, we arrive at the fundamental decomposition:
\begin{equation}\label{eq:log_part_decomp}
    \log Z(x) = \mathcal{F}(q_\phi) + \KL(q_\phi \parallel \pi^*).
\end{equation}

Because $\log Z(x)$ is a constant with respect to $\phi$, this derivation proves that \textbf{maximizing the tractable lower bound $\mathcal{F}(q_\phi)$ is mathematically equivalent to minimizing the divergence $\KL(q_\phi \parallel \pi^*)$.} This establishes the following E-Step optimization:

\begin{tcolorbox}[colback=gray!5!white, colframe=gray!50!black, title=\textbf{E-Step: Teacher Refinement}]
We optimize the teacher parameters $\phi$ to maximize $\mathcal{F}(q_\phi)$ (effectively minimizing $\KL(q_\phi \parallel \pi^*)$). As derived in Section \ref{subsec:e_step}, we approximate this optimization using Binary Classifier Optimization on the implicit rewards.
\end{tcolorbox}

To mathematically justify the M-step, we must connect this updated teacher back to our global objective: minimizing $\KL(\pi_\theta \parallel \pi^*)$. By adding and subtracting $\log q_\phi(y \mid x, \mathcal{C})$ inside the expectation of this global objective, we can decompose it as follows:
\begin{align}\label{eq:m_step_decomp}
    \KL(\pi_\theta \parallel \pi^*) &= \mathbb{E}_{y \sim \pi_\theta} \left[ \log \pi_\theta(y \mid x) - \log \pi^*(y \mid x) \right] \nonumber \\
    &= \mathbb{E}_{y \sim \pi_\theta} \left[ \log \pi_\theta(y \mid x) - \log q_\phi(y \mid x, \mathcal{C}) + \log q_\phi(y \mid x, \mathcal{C}) - \log \pi^*(y \mid x) \right] \nonumber \\
    &= \KL(\pi_\theta \parallel q_\phi) + \mathbb{E}_{y \sim \pi_\theta} \left[ \log \frac{q_\phi(y \mid x, \mathcal{C})}{\pi^*(y \mid x)} \right].
\end{align}

While the E-step cannot guarantee that $q_\phi$ perfectly matches the optimal policy $\pi^*$ due to finite network capacity, it actively minimizes the residual error represented by the second term. By treating the optimized $q_\phi$ as the tightest tractable surrogate for the optimal policy, the M-step circumvents the intractable global objective by instead minimizing the proxy divergence $\KL(\pi_\theta \parallel q_\phi)$. This establishes the following M-Step optimization:

\begin{tcolorbox}[colback=gray!5!white, colframe=gray!50!black, title=\textbf{M-Step: Student Distillation}]
Operating on the mathematical decomposition above, we project this refined knowledge back into the student policy space. By minimizing $\KL(\pi_\theta \parallel q_\phi)$ over the student's on-policy rollouts, we strictly minimize the remaining divergence against our optimal surrogate, completing the distillation cycle without requiring the privileged feedback $\mathcal{C}$ at inference time.
\end{tcolorbox}

\subsection{Theoretical Implications of the Dynamic Reference Model}\label{app:dynamic_prior}

As established in Equation \ref{eq:effective_implicit_reward}, we define the effective implicit reward using a dynamic prior by setting the reference model to the current student policy, $\pi_\theta$. This formulation is mathematically equivalent to defining an optimal target distribution $\pi^*_{\text{dyn}}$ where the student acts as the base prior:
\begin{equation}
    \pi^*_{\text{dyn}}(y \mid x, \mathcal{C}) = \frac{1}{Z_{\text{dyn}}(x)} \pi_\theta(y \mid x) \exp \left( \frac{r(x, y)}{\beta} \right)
\end{equation}
where $Z_{\text{dyn}}(x)$ is the partition function. 

To deeply understand the regularizing effect of this dynamic prior, we analyze the E-step objective from two complementary mathematical perspectives: its relationship to standard static priors, and its role as an adaptive trust region.

\textbf{Perspective 1: The Adaptive Alignment Bonus.} First, we can contrast our dynamic target directly against the standard optimal target $\pi^*$ introduced previously in Equation \ref{eq:optimal_policy}, which utilizes a fixed, static prior anchored to the initial reference model $\pi_{\text{ref}}$. We can analytically express our new dynamic target as a probability-ratio scaled version of this static target:
\begin{equation}
    \pi^*_{\text{dyn}}(y \mid x, \mathcal{C}) = \left( \frac{Z_{\text{stat}}(x)}{Z_{\text{dyn}}(x)} \right) \frac{\pi_\theta(y \mid x)}{\pi_{\text{ref}}(y \mid x)} \pi^*(y \mid x)
\end{equation}

Because our E-step objective is to minimize the KL divergence between the teacher policy $q_\phi$ and this dynamic optimal target, we can substitute the ratio expansion into the objective:
\begin{align}
    \min_\phi D_{\text{KL}}(q_\phi \parallel \pi^*_{\text{dyn}}) &= \mathbb{E}_{q_\phi} \left[ \log \frac{q_\phi(y \mid x, \mathcal{C})}{\pi^*_{\text{dyn}}(y \mid x, \mathcal{C})} \right] \nonumber \\
    &= \mathbb{E}_{q_\phi} \left[ \log \frac{q_\phi}{\pi^*} - \log \frac{\pi_\theta}{\pi_{\text{ref}}} - \log \frac{Z_{\text{stat}}}{Z_{\text{dyn}}} \right] \nonumber \\
    &= D_{\text{KL}}(q_\phi \parallel \pi^*) - \mathbb{E}_{q_\phi} \left[ \log \frac{\pi_\theta(y \mid x)}{\pi_{\text{ref}}(y \mid x)} \right] + C.
\end{align}

This derivation reveals a profound insight into how VPD stabilizes co-evolution. The E-step objective is mathematically equivalent to the standard, static objective, but with a critical additional penalty term: $-\mathbb{E}_{q_\phi} [\log (\pi_\theta / \pi_{\text{ref}})]$. This extra term acts as an adaptive alignment bonus. It explicitly incentivizes the teacher to upweight trajectories where the current student ($\pi_\theta$) assigns a higher likelihood than the base model ($\pi_{\text{ref}}$). Instead of rigidly pulling the teacher toward an absolute global optimum that may be conceptually out of reach for the student, the teacher actively anchors its guidance to the pathways the student has already begun to master. 

\textbf{Perspective 2: The Sliding Trust Region.} Alternatively, to explicitly demonstrate how this dynamic prior restricts the teacher's updates, we can expand the E-step KL divergence directly:
\begin{align}
    \min_\phi D_{\text{KL}}(q_\phi \parallel \pi^*_{\text{dyn}}) &= \mathbb{E}_{q_\phi} \left[ \log \frac{q_\phi(y \mid x, \mathcal{C})}{\pi^*_{\text{dyn}}(y \mid x, \mathcal{C})} \right] \nonumber \\
    &= \mathbb{E}_{q_\phi} \left[ \log q_\phi(y \mid x, \mathcal{C}) - \log \left( \frac{1}{Z_{\text{dyn}}(x)} \pi_\theta(y \mid x) \exp\left(\frac{r(x, y)}{\beta}\right) \right) \right] \nonumber \\
    &= \mathbb{E}_{q_\phi} \left[ \log \frac{q_\phi(y \mid x, \mathcal{C})}{\pi_\theta(y \mid x)} - \frac{r(x, y)}{\beta} + \log Z_{\text{dyn}}(x) \right] \nonumber \\
    &= D_{\text{KL}}(q_\phi \parallel \pi_\theta) - \frac{1}{\beta} \mathbb{E}_{q_\phi} [r(x, y)] + \log Z_{\text{dyn}}(x).
\end{align}

Because $\log Z_{\text{dyn}}(x)$ is a constant with respect to the teacher's parameters $\phi$, minimizing this KL divergence is mathematically identical to solving the following constrained optimization problem:
\begin{equation}\label{eq:trust_region}
    \max_\phi \mathbb{E}_{q_\phi} [r(x, y)] - \beta D_{\text{KL}}(q_\phi \parallel \pi_\theta).
\end{equation}

This direct expansion exposes the profound regularizing effect of the dynamic prior. By anchoring the implicit reward to the current student $\pi_\theta$, the E-step objective naturally produces a strict KL penalty against the student's active distribution. Our dynamic formulation fundamentally acts as a sliding trust region. It enforces a geometric boundary: the teacher is permitted to adjust its distribution to maximize the diagnostic language feedback, but it is mathematically penalized for proposing targets that fall outside the student's current active learning capacity. This ensures the dense distributional guidance provided during the M-step is always a safely reachable "next step," thereby eliminating the catastrophic training collapse observed in standard single-phase hybrid methods.

\textbf{Stable Inner Loops and Dynamic Outer Loops.} Finally, it is important to distinguish the stationary nature of the E-step from the global trajectory of training. Because the student parameters $\theta$ are continuously updated during the M-step, the optimal target $\pi^*_{\text{dyn}}$ dynamically shifts across the overall training process. However, the reference likelihoods $\pi_\theta(y \mid x)$ are pre-calculated and frozen before each individual E-step begins. This ensures that within the bounds of any specific E-step, the target remains strictly stationary, allowing for highly stable gradient descent while still supporting the continuous co-evolution of the policies.

\section{Method}\label{sec:method_app}

\subsection{VPD Algorithm}

\begin{algorithm}[h]
\caption{Variational Policy Distillation via Co-Evolutionary EM}
\label{alg:em_vpd}
\begin{algorithmic}[1]
\REQUIRE Initialized model parameters $\theta_0$, Prompt dataset $\mathcal{D}$, Verifiable Environment $\mathcal{E}$, KL coefficient $\beta$, Number of iterations $K$
\FOR{$k = 1, 2, \dots, K$}
    \STATE \textbf{// 1. On-Policy Rollout}
    \STATE Sample a batch of prompts $x \sim \mathcal{D}$.
    \STATE Generate student trajectories $y \sim \pi_{\theta_{k-1}}(\cdot \mid x)$.
    
    \STATE \textbf{// 2. Environment Critique}
    \FOR{each trajectory $y$}
        \STATE Evaluate $y$ using $\mathcal{E}$ to obtain binary outcome $r(x,y)$ and diagnostic feedback $\mathcal{C}$.
    \ENDFOR
    \STATE Partition trajectories and feedback into successes $\{(y^+, \mathcal{C}^+)\}$ and failures $\{(y^-, \mathcal{C}^-)\}$
    
    \STATE \textbf{// 3. E-Step: Teacher Refinement (Unpaired Preference)}
    \STATE Synchronize teacher alias with current student weights: $\phi_k \leftarrow \theta_{k-1}$.
    \STATE Compute the dynamic reward shift $\delta = \frac{1}{2}(\mathbb{E}[\tilde{r}_{\phi_k}(x, y^+, \mathcal{C}^+)] + \mathbb{E}[\tilde{r}_{\phi_k}(x, y^-, \mathcal{C}^-)])$.
    \STATE Update parameters to obtain refined teacher $\phi'_k$ by minimizing $\mathcal{L}_{\text{E-step}}$ (Eq.~\ref{eq:e_step_final}).
    
    \STATE \textbf{// 4. M-Step: Student Distillation (Information Projection)}
    \STATE Update student $\theta_k$ by minimizing $\mathcal{L}_{\text{M-step}}$ against the stop-gradient teacher $\text{sg}[q_{\phi'_k}]$ (Eq.~\ref{eq:m_step_final}).
\ENDFOR
\RETURN Optimized policy $\theta_K$
\end{algorithmic}
\end{algorithm}

\subsection{Convergence Properties of the VPD EM Procedure}
\label{app:convergence}

A natural concern with any EM-style algorithm is whether alternating E- and M-steps yields monotonic progress toward the global objective. We formalize this below and connect it to the practical setting where each step performs only approximate gradient updates.

\begin{proposition}[Monotonic ELBO Improvement]\label{prop:monotone}
Let $\theta_k$ denote the student parameters at the start of iteration $k$. Suppose the E-step and M-step each attain their respective global optima within iteration $k$. Then the RLVR objective $\mathcal{J}(\theta_k)$ (Eq.~\ref{eq:rlvr}) is non-decreasing across iterations: $\mathcal{J}(\theta_{k+1}) \geq \mathcal{J}(\theta_k)$.
\end{proposition}

\begin{proof}
We chain two results already established in Appendix~\ref{sec:theory_app}.

\textbf{Step 1 (E-step improves the bound).}
From Eq.~\ref{eq:log_part_decomp}, the log-partition function decomposes as $\log Z_{\mathrm{dyn}}(x) = \mathcal{F}(q_\phi) + D_{\mathrm{KL}}(q_\phi \| \pi^*_{\mathrm{dyn}})$, where $\mathcal{F}(q_\phi) = \frac{1}{\beta}\mathbb{E}_{q_\phi}[r(x,y)] - D_{\mathrm{KL}}(q_\phi \| \pi_{\theta_k})$ is the dynamic ELBO. Since $\log Z_{\mathrm{dyn}}(x)$ is constant w.r.t.\ $\phi$, any reduction in $D_{\mathrm{KL}}(q_\phi \| \pi^*_{\mathrm{dyn}})$ produces an equal increase in $\mathcal{F}$. The E-step optimizes exactly this divergence, so the refined teacher $q_{\phi'_k}$ satisfies $\mathcal{F}(q_{\phi'_k}) \geq \mathcal{F}(q_{\theta_k})$.

\textbf{Step 2 (M-step improves the objective).}
From Eq.~\ref{eq:m_step_decomp}, the global divergence decomposes as $D_{\mathrm{KL}}(\pi_\theta \| \pi^*) = D_{\mathrm{KL}}(\pi_\theta \| q_{\phi'_k}) + \mathbb{E}_{\pi_\theta}[\log(q_{\phi'_k}/\pi^*)]$. Holding $q_{\phi'_k}$ fixed, the second term is constant w.r.t.\ $\theta$. The M-step minimizes $D_{\mathrm{KL}}(\pi_\theta \| q_{\phi'_k})$, directly reducing $D_{\mathrm{KL}}(\pi_\theta \| \pi^*)$. By the equivalence in Eq.~\ref{eq:kl_to_rlvr}, this implies $\mathcal{J}(\theta_{k+1}) \geq \mathcal{J}(\theta_k)$.

\textbf{Step 3 (Convergence).}
Chaining Steps~1 and~2 yields a monotonically non-decreasing sequence $\{\mathcal{J}(\theta_k)\}_{k=1}^K$. Since $\mathcal{J}$ is upper-bounded by $\beta \log Z(x)$ (Eq.~\ref{eq:log_part_bound}), the sequence converges.
\end{proof}

\paragraph{Extension to approximate updates (Generalized EM).}
In practice, neither step is solved to global optimality; each performs a finite number of gradient updates. Following the Generalized EM framework~\citep{neal1998view}, monotonic ELBO improvement is preserved provided each step \emph{does not decrease} its respective sub-objective---a condition satisfied by gradient descent with a sufficiently small learning rate. Two features of VPD further strengthen this practical guarantee:(i)~the dynamic trust-region (Eq.~\ref{eq:trust_region}) explicitly penalizes large teacher drift via $D_{\mathrm{KL}}(q_\phi \| \pi_{\theta_k})$, bounding per-iteration parameter change; and (ii)~freezing the student likelihoods $\pi_{\theta_k}(y|x)$ during each E-step ensures a stationary inner-loop landscape. Together, these place VPD squarely within the Generalized EM family with bounded per-step updates, providing strong practical assurance of monotonic progress even under partial optimization.

\subsection{Hybrid Distillation and RL}
\label{subsec:hybrid_baselines}

While our EM framework elegantly separates the integration of verifiable rewards (E-step) from the distillation of the reasoning process (M-step), an alternative paradigm in recent literature attempts to fuse these signals simultaneously within a single optimization phase. Since the on-policy nature of our M-step evaluates trajectories against a ground-truth verifier at no additional sampling cost, it is technically possible to compress both the dense distillation signal and the sparse scalar advantage into a single update. To rigorously benchmark the benefits of our decoupled co-evolutionary approach, we formulate three representative single-phase hybrid mechanisms to serve as strong baselines:

\textbf{1. Joint Loss Optimization.}
The most straightforward baseline computes the objective losses independently and optimizes their weighted sum. We calculate the standard GRPO surrogate loss $\mathcal{L}_{\text{GRPO}}$ using the sequence-level advantages, and combine it with the SDPO KL distillation loss:
\begin{equation}
    \mathcal{L}_{\text{Hybrid}}(\theta) = \omega_{\text{opd}} \cdot \mathcal{L}_{\text{SDPO}}(\theta) + \omega_{\text{rl}} \cdot \mathcal{L}_{\text{GRPO}}(\theta),
\end{equation}
where $\omega_{\text{opd}}$ and $\omega_{\text{rl}}$ are hyperparameters balancing the dense distributional guidance of the teacher with the unbiased terminal outcome of the verifier. While simple to implement, linearly combining a bounded reward with a potentially high-variance KL loss can often cause scale mismatches.

\textbf{2. Advantage Reshaping.}
Instead of summing the final losses, a second class of baselines fuses the signals at the advantage level. Following the methodology of Self-Distillation Policy Optimization (SDPO) \cite{hubotter2026reinforcement}, the teacher's dense distillation signal can be translated into a per-token advantage, $A_t^{\text{SDPO}} = \text{sg} \left( \log q_\phi(y_t \mid x, \mathcal{C}, y_{<t}) - \log \pi_\theta(y_t \mid x, y_{<t}) \right)$. This is then linearly combined with the sequence-level Monte Carlo advantage provided by the verifier, $A^{\text{GRPO}}$. For a given token $y_t$, the fused advantage is:
\begin{equation}
    A_t^{\text{Hybrid}}(y_t) = \omega_{\text{rl}} \cdot A^{\text{GRPO}} + \omega_{\text{opd}} \cdot A_t^{\text{SDPO}}(y_t).
\end{equation}
The combined advantage is then used as a direct drop-in replacement in the standard clipped policy gradient objective:
\begin{equation}
    \mathcal{L}_{\text{PPO}}(\theta) = - \mathbb{E}_{x, y} \left[ \sum_{t=1}^{|y|} \min \left( \rho_t(\theta) A_t^{\text{Hybrid}}(y_t), \text{clip}(\rho_t(\theta), 1-\epsilon, 1+\epsilon) A_t^{\text{Hybrid}}(y_t) \right) \right],
\end{equation}
where $\rho_t(\theta) = \pi_\theta(y_t \mid x, y_{<t}) / \pi_{\theta_{\text{old}}}(y_t \mid x, y_{<t})$ is the token-level importance sampling ratio. This baseline fundamentally balances the unbiased, high-variance nature of the environment reward with the dense, low-variance bootstrapped signal from the self-teacher.

\textbf{3. Distillation-Guided Advantage Reweighting.}
A fundamental limitation of the standard GRPO advantage $A^{\text{GRPO}}$ is its uniform application to all tokens in a sequence, failing to differentiate between critical reasoning steps and generic filler. To construct a baseline that addresses this without fully decoupling the steps, we can explicitly weight the sequence-level advantage using the teacher's distillation signal.

First, we capture the token-level discrepancy between the teacher and student using the log-ratio:
\begin{equation}
    \Delta_t = \text{sg} \left( \log q_\phi(y_t \mid x, \mathcal{C}, y_{<t}) - \log \pi_\theta(y_t \mid x, y_{<t}) \right).
\end{equation}
To rigorously formulate this multiplier and avoid destabilizing the RL training with unbounded log-ratios, we adapt the mathematical mechanics of RLVR with Self-Distillation (RLSD) \citep{yang2026self}. Specifically, we construct an exponential weight modulated by the sign of the sequence-level advantage:
\begin{equation}
    w_t = \exp(\text{sign}(A^{\text{GRPO}}) \cdot \Delta_t).
\end{equation}
This ensures that when a trajectory is successful ($A^{\text{GRPO}} > 0$), tokens favored by the teacher receive amplified credit, whereas in a failed trajectory ($A^{\text{GRPO}} < 0$), tokens disfavored by the teacher bear a greater penalty. 

To safely integrate this reweighting, we adapt RLSD's stabilization formula, applying a trust-region clip $\epsilon_w$ and a decaying mixing coefficient $\alpha \in [0, 1]$:
\begin{equation}
    A_t^{\text{Hybrid}}(y_t) = A^{\text{GRPO}} \cdot \left( (1 - \alpha) + \alpha \cdot \text{clip}(w_t, 1 - \epsilon_w, 1 + \epsilon_w) \right).
\end{equation}
By decaying $\alpha$ over the course of training, this baseline relies heavily on the teacher's fine-grained magnitude guidance early on, and smoothly transitions to pure verifier-driven GRPO as the policy matures. The hybrid advantage is similarly inserted into the standard policy gradient objective. This formulation ensures the overall update direction remains strictly anchored to the environment verifier, while our distillation signal dynamically reshapes the token-level credit.

\section{Experiments}\label{sec:experiment_app}

\subsection{Environment Feedback (LiveCodeBench)}\label{subsec:lcb}

\textbf{Dataset and Evaluation Protocol.} 
We evaluate our framework on the LiveCodeBench (LCB) v6 subset \citep{jain2024livecodebench}. To ensure a rigorous evaluation of generalization, we follow the exact protocol established by SDPO \citep{hubotter2026reinforcement}: during the on-policy rollout generation, the environment verifies the model's code against the \emph{public} unit tests provided in the prompt. If the code fails, the specific compiler error or assertion failure is returned as the diagnostic feedback $\mathcal{C}$. However, final model performance is exclusively evaluated on the held-out \emph{private} unit tests. Following the SDPO evaluation standard, we report the average accuracy over 4 independent rollouts to reliably measure the model's robust reasoning capabilities and ensure it is not merely overfitting to the visible test cases.

\textbf{Model Configuration.} 
We adopt Qwen3-8B \citep{yang2025qwen3} as our base model. To rigorously isolate the impact of our VPD framework and prevent the policy from relying on built-in extended reasoning mechanisms, we explicitly disable Qwen3's ``thinking mode'' across all evaluations. Additionally, we construct the feedback-augmented prompt for the teacher policy using the exact template format established by SDPO \citep{hubotter2026reinforcement}.

\textbf{Hyperparameters.} 
Across all baselines and VPD, we sample $N=8$ on-policy rollouts per prompt. For VPD, we apply an asymmetric update frequency, performing one E-step optimization for every $F=5$ M-step gradient updates. The shared-weight network is optimized using the AdamW optimizer. Detailed hyperparameters are provided in Table~\ref{tab:lcb_hyperparams}.

\begin{table}[h]
    \centering
    \caption{Hyperparameters for LiveCodeBench experiments (Qwen3-8B).}
    \label{tab:lcb_hyperparams}
    \small
    \begin{tabular}{lc}
    \toprule
    \textbf{Hyperparameter} & \textbf{Value} \\
    \midrule
    Learning Rate & $1 \times 10^{-6}$ \\
    Epochs & 30 \\
    Rollout Batch Size & 32 \\
    Rollouts per Prompt ($N$) & 8 \\
    Max Prompt Length & 2048 \\
    Max Response Length & 8192 \\
    PPO Minibatch Size & 1 \\
    SDPO Logits Topk & 20 \\
    SDPO Loss & Reverse KL \\
    SDPO Teacher Update Rate & 0.01 \\
    VPD E-Step Frequency ($F$) & 5 \\
    VPD E-Step Minibatch Size & 32 \\
    BCO temperature ($\beta$) & 0.1 \\
    \bottomrule
    \end{tabular}
\end{table}

\subsection{Contrastive Sibling Rollouts as Feedback (SciKnowEval)}\label{subsec:sciqa}

\textbf{Dataset and Evaluation Protocol.} 
We evaluate the generalization of our framework across diverse scientific domains using the SciKnowEval benchmark \citep{feng2024sciknoweval}, encompassing Biology, Chemistry, Materials Science, and Physics. Unlike code generation environments, SciKnowEval does not natively provide dense execution tracebacks; instead, the environment acts as a sparse verifier, returning only a binary correctness score via exact-match grading of the final extracted answer. We report the final average accuracy (Avg@16) on the benchmark's evaluation split.

\textbf{Contrastive Sibling Feedback Mechanism.} 
Because the environment cannot generate natural language critiques, we synthesize the diagnostic feedback $\mathcal{C}$ using the model's own successful rollouts. During the rollout phase, we sample $N$ trajectories per prompt. For each generated rollout, we select one successful trajectory from the same prompt's rollout group (excluding itself) to serve as the conditioning information for the teacher. If all $N$ rollouts for a given prompt are uniformly incorrect (meaning no successful sibling exists), those samples are entirely excluded from both the E-step and M-step updates. The prompt template for the teacher policy strictly follows the exact formatting established in SDPO.

\textbf{Model Configuration and Hyperparameters.} 
We evaluate Qwen3-1.7B, Qwen3-8B \citep{yang2025qwen3}, and OLMo3-7B-Instruct \citep{olmo2025olmo}. As with our LiveCodeBench experiments, we explicitly disable the built-in ``thinking mode'' for all models to rigorously isolate the benefits of the VPD framework. We sample $N=8$ trajectories per prompt during the data collection phase. The shared-weight network is optimized using the AdamW optimizer. To maintain stable trust-region dynamics during the EM cycle, we employ an asymmetric update frequency of $F=5$ (one E-step per five M-steps). Detailed hyperparameters are provided in Table~\ref{tab:sciknoweval_hyperparams}.

\begin{table}[ht]
    \centering
    \caption{Hyperparameters for SciKnowEval experiments.}
    \label{tab:sciknoweval_hyperparams}
    \small
    \begin{tabular}{lc}
    \toprule
    \textbf{Hyperparameter} & \textbf{Value} \\
    \midrule
    Learning Rate & $1 \times 10^{-5}$ (Qwen) / $5 \times 10^{-6}$ (OLMO) \\
    Training Steps & 500 \\
    Rollout Batch Size & 32 \\
    Rollouts per Prompt ($N$) & 8 \\
    Max Prompt Length & 2048 \\
    Max Response Length & 8192 \\
    PPO Minibatch Size & 32 \\
    SDPO Logits Top-k & 100 \\
    SDPO Loss & JS \\
    SDPO Teacher Update Rate & 0.05 \\
    VPD E-Step Frequency ($F$) & 5 \\
    VPD E-Step Minibatch Size & 32 \\
    BCO Temperature ($\beta$) & 0.1 \\
    \bottomrule
    \end{tabular}
\end{table}

\subsection{Self-Critique via LLM Judge (SciKnowEval)}\label{subsec:judge}

\textbf{Motivation and Protocol.} 
As discussed in Section \ref{subsec:sciqa}, the contrastive sibling mechanism relies on the presence of at least one successful trajectory within a prompt's rollout group. If all $N$ rollouts fail, the sample provides no contrastive textual guidance. To evaluate our framework's ability to operate without contrastive positive pairs, we introduce an autonomous LLM Judge setting. In this configuration, the model acts as its own evaluator, generating diagnostic critiques for its failed trajectories. 

\textbf{Critique Generation Prompt.} 
To generate the diagnostic feedback $\mathcal{C}$, the failed trajectory is passed back to the model along with the ground-truth answer. Because SciKnowEval acts as a sparse verifier, it does not provide step-by-step ground-truth derivation; the available \texttt{\{reference\_solution\}} is strictly limited to the final multiple-choice letter (e.g., ``The correct answer is A''). The model must therefore independently deduce the logical gap between its attempt and the correct final answer. We generate the critique using the following zero-shot prompt:

\begin{tcolorbox}[colback=gray!5,colframe=gray!50,title=LLM Judge Critique Prompt]
\small
\texttt{Given this multiple-choice science question, a student's incorrect attempt, and the correct answer, provide specific feedback on what went wrong in the student's reasoning.} \\

\texttt{Problem: \{problem\}} \\

\texttt{Student's attempt: \{response\}} \\

\texttt{Reference solution: \{reference\_solution\}} \\

\texttt{Provide 2-3 sentences of specific, actionable feedback explaining why the student's chosen answer is wrong and what conceptual misunderstanding led to the error. Do not simply state the correct answer.}
\end{tcolorbox}

\textbf{Teacher Conditioning and Hyperparameters.} 
Once the critique is generated, it serves as the diagnostic feedback $\mathcal{C}$ for the E-step. The teacher policy is conditioned on this critique using a format identical to the contrastive sibling template (replacing the sibling trajectory with the generated text). The subsequent E-step and M-step optimizations proceed exactly as in the contrastive setting. All model configurations, rollout strategies ($N=8$), update frequencies ($F=5$), and optimization hyperparameters are identical to those detailed in Table~\ref{tab:sciknoweval_hyperparams}.

\subsection{Base Model ``Cold Start''}
\label{subsec:cold_start}

\textbf{Motivation and Setup.} 
To investigate the reliance of self-distillation methods on pre-existing instruction-following capabilities, we evaluate our framework in a pure ``cold start'' scenario. In this setting, we directly optimize a base foundation model using on-policy generation. We utilize Qwen3-4B-Base and evaluate its performance on the SciKnowEval benchmark. 

\textbf{Prompting and Formatting.} 
We maintain the exact same prompt templates used for the instruction-tuned models. However, because the base model lacks instruction-following alignment, it struggles to efficiently parse and integrate the textual critique ($\mathcal{C}$) into its subsequent reasoning steps. For standard self-distillation, this inability to utilize feedback causes the teacher's target distribution to rapidly degrade.

\textbf{Hyperparameters and Stability.} 
All hyperparameters are kept identical to those detailed in Table~\ref{tab:sciknoweval_hyperparams}. While pure RL (GRPO) successfully elicits reasoning capabilities through sparse reward maximization, standard SDPO immediately collapses to a 0\% pass rate within the first few steps. By actively optimizing the teacher policy and constraining the trust region, VPD's E-step effectively prevents this rapid collapse, maintaining a functional learning trajectory for substantially longer. However, because the M-step distillation inherently relies on tracking the teacher's intermediate token distribution, the base model still ultimately struggles compared to pure outcome-based RL, resulting in lower final accuracy. The results are summarized in Table~\ref{tab:cold_start_app}.

\begin{table}[ht]
    \centering
    \caption{Cold start training on SciKnowEval using Qwen3-4B-Base.}
    \label{tab:cold_start_app}
    \footnotesize
    \setlength{\tabcolsep}{5pt} 
    \begin{tabular}{lccccc}
    \toprule
    \textbf{Method} & \textbf{Bio.} & \textbf{Chem.} & \textbf{Mat.} & \textbf{Phys.} & \textbf{AVG} \\
    \midrule
        Base & 29.37 & 32.71 & 50.84 & 51.33 & 41.06 \\
        GRPO & 58.63 & 77.98 & 81.12 & 80.23 & 74.49 \\
        VPD & 48.50 & 70.45 & 64.03 & 72.81 & 63.95 \\
    \bottomrule
    \end{tabular}
\end{table}

Figure~\ref{fig:cold_start_app} presents additional training curves for the base model cold-start experiments across the remaining SciKnowEval domains.

\begin{figure}[h]
    \centering
    \includegraphics[width=0.3\linewidth]{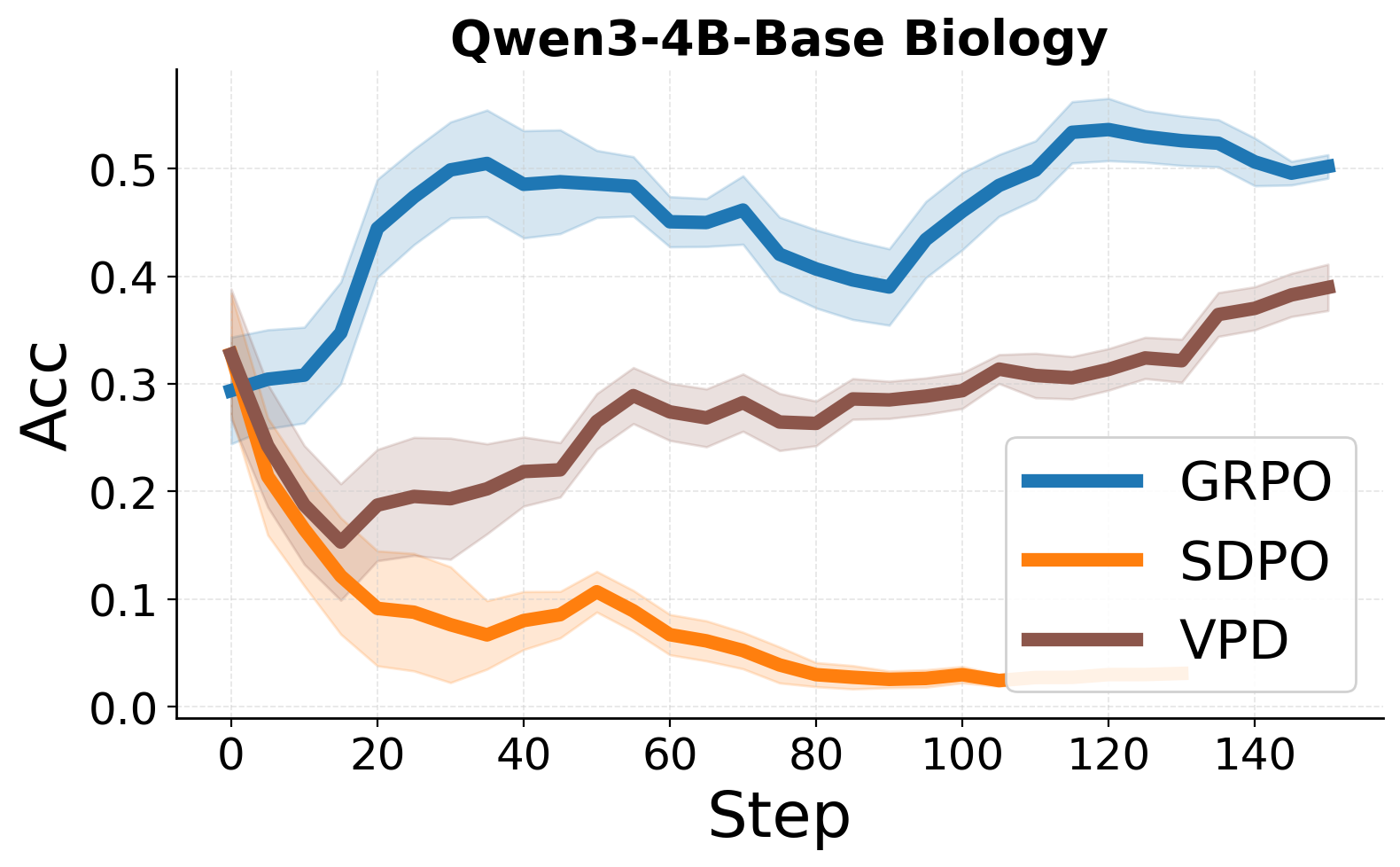}
    \includegraphics[width=0.3\linewidth]{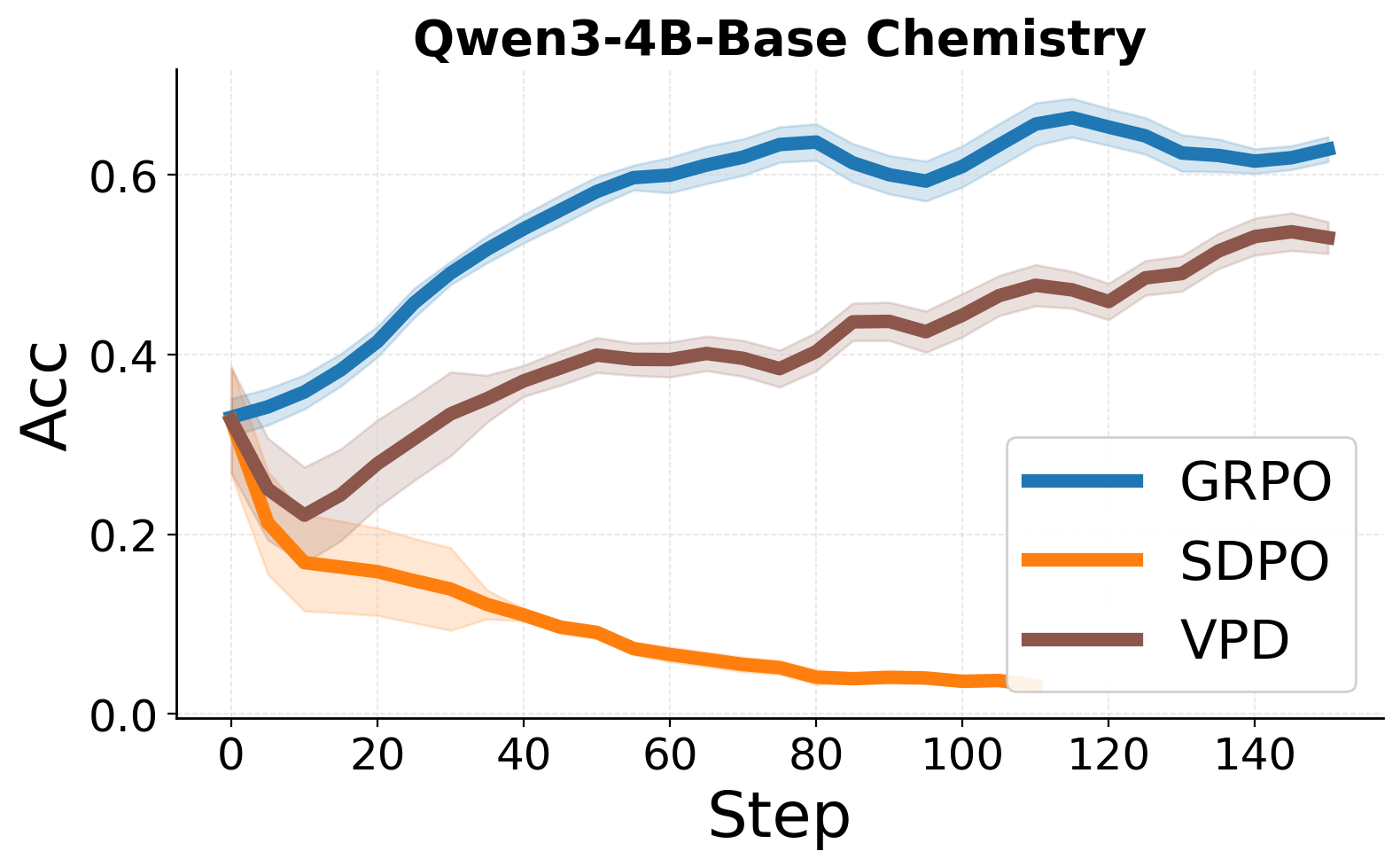}
    \includegraphics[width=0.3\linewidth]{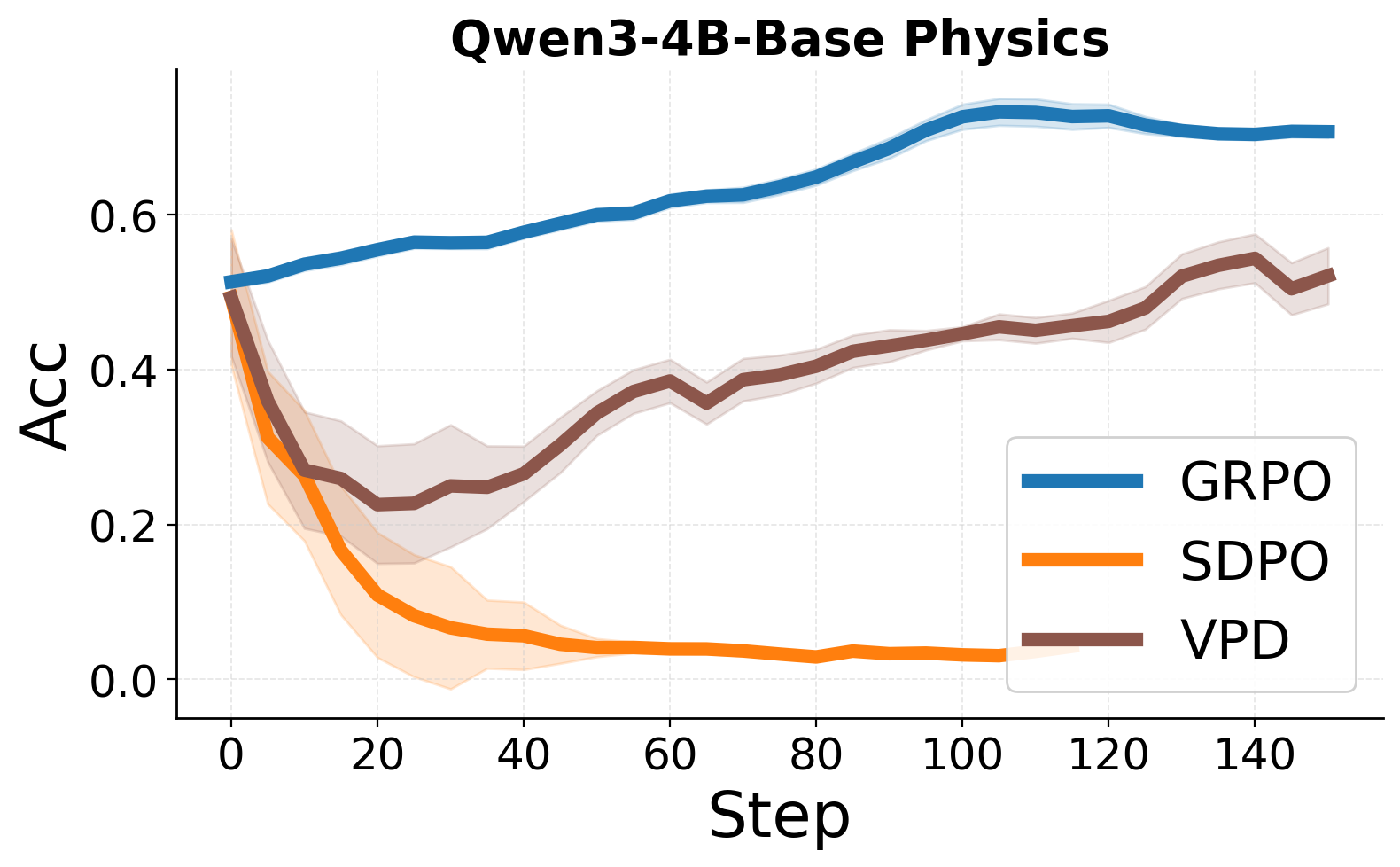}
\caption{Additional training curves illustrating the cold-start optimization dynamics of Qwen3-4B-Base across the Biology, Chemistry, and Physics domains. While standard SDPO collapses almost immediately, VPD actively stabilizes the trust region and effectively delays training degradation.}
\label{fig:cold_start_app}
\end{figure}

\subsection{Mathematical Reasoning Experimental Setup}
\label{subsec:math_reasoning}

\textbf{Dataset and Evaluation Protocol.} 
To stress-test our framework on highly symbolic and strict logical domains, we utilize the DAPO-Math dataset \citep{yu2025dapo} for training and evaluate the resulting policies on the Math500 benchmark \citep{lightman2023lets}. DAPO-Math provides a challenging distribution of competitive mathematics problems. During training, the verifiable environment parses the model's final answer (enclosed in a \texttt{\textbackslash boxed\{\}} tag) and uses a symbolic exact-match grader to assign a binary reward. Final performance is evaluated using the standard Math500 test script, reporting the Avg@4 accuracy.

\textbf{Feedback Mechanism for Mathematics.} 
For mathematical reasoning, we employ the contrastive sibling mechanism to generate the diagnostic feedback $\mathcal{C}$. If a model's logical derivation leads to an incorrect final boxed answer, the teacher policy is provided with a successful sibling trajectory to identify the algebraic or conceptual error. 

\textbf{Optimization Dynamics.} 
As noted in the main text, mathematical reasoning is uniquely unforgiving; intermediate self-distillation targets ($\pi^*$) that contain subtle arithmetic hallucinations can actively poison the student policy. While pure sparse RL (GRPO) achieves an 83.8\% accuracy by strictly rewarding fully correct derivations, standard SDPO suffers from severe training collapse as it distills these flawed linguistic explanations. We sample $N=8$ trajectories per prompt and apply an asymmetric update frequency ($F=5$) for VPD. Detailed training hyperparameters for the DAPO-Math experiments are provided in Table~\ref{tab:math_hyperparams}.

\begin{table}[ht]
    \centering
    \caption{Hyperparameters for DAPO-Math experiments (Qwen3-8B).}
    \label{tab:math_hyperparams}
    \small
    \begin{tabular}{lc}
    \toprule
    \textbf{Hyperparameter} & \textbf{Value} \\
    \midrule
    Learning Rate & $5 \times 10^{-6}$ \\
    Training Steps & 200 \\
    Rollout Batch Size & 32 \\
    Rollouts per Prompt ($N$) & 8 \\
    Max Prompt Length & 2048 \\
    Max Response Length & 8192 \\
    PPO Minibatch Size & 32 \\
    SDPO Logits Top-k & full logits \\
    SDPO Loss & Forward KL \\
    SDPO Teacher Update Rate & 0.05 \\
    VPD E-Step Frequency ($F$) & 5 \\
    VPD E-Step Minibatch Size & 32 \\
    BCO Temperature ($\beta$) & 0.1 \\
    \bottomrule
    \end{tabular}
\end{table}

Figure~\ref{fig:math_app} presents additional training curves for the mathematical reasoning experiments across the AIME24, AIME25, and AMC 23 benchmarks.

\begin{figure}[h]
    \centering
    \includegraphics[width=0.32\linewidth]{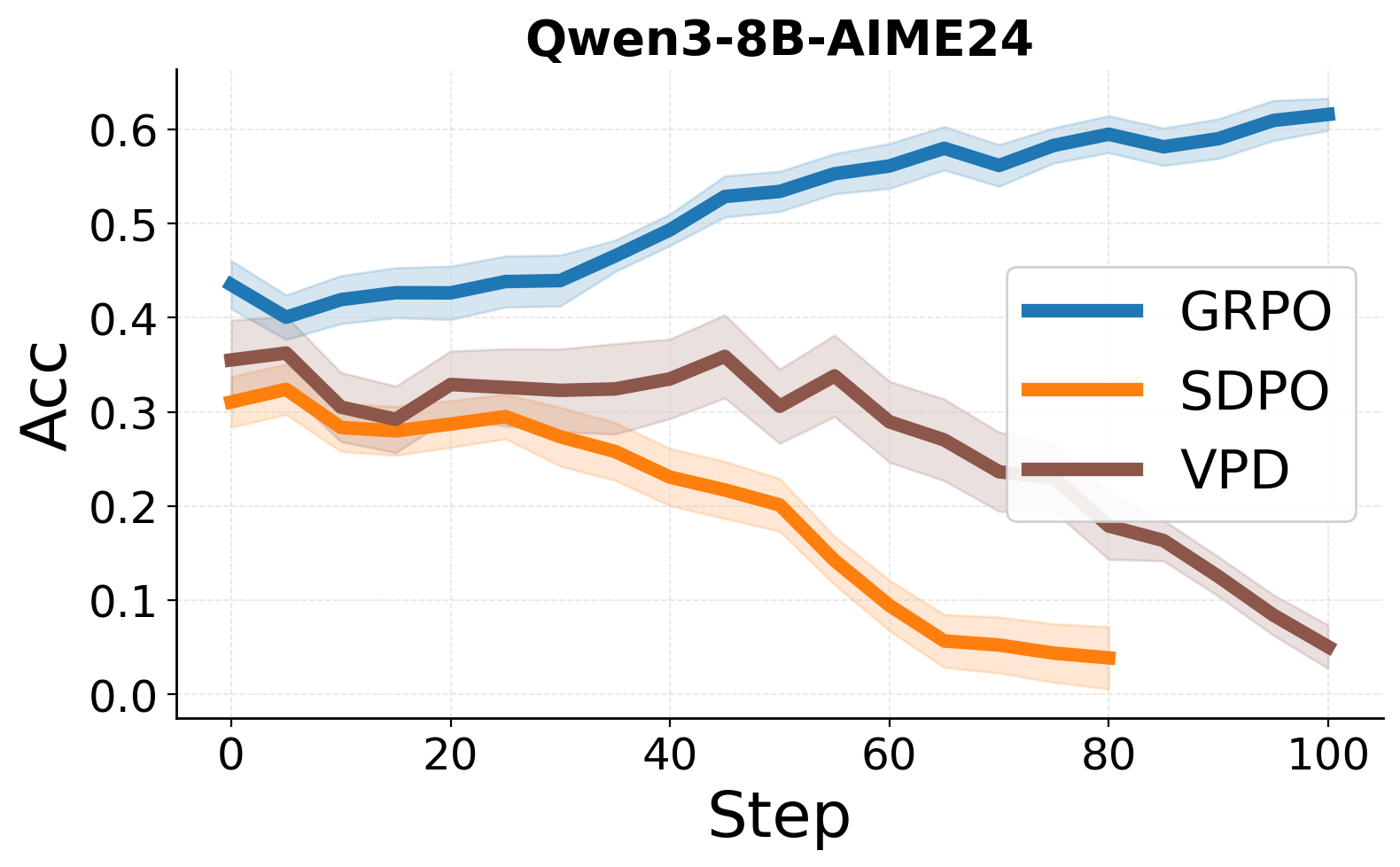}\hfill
    \includegraphics[width=0.32\linewidth]{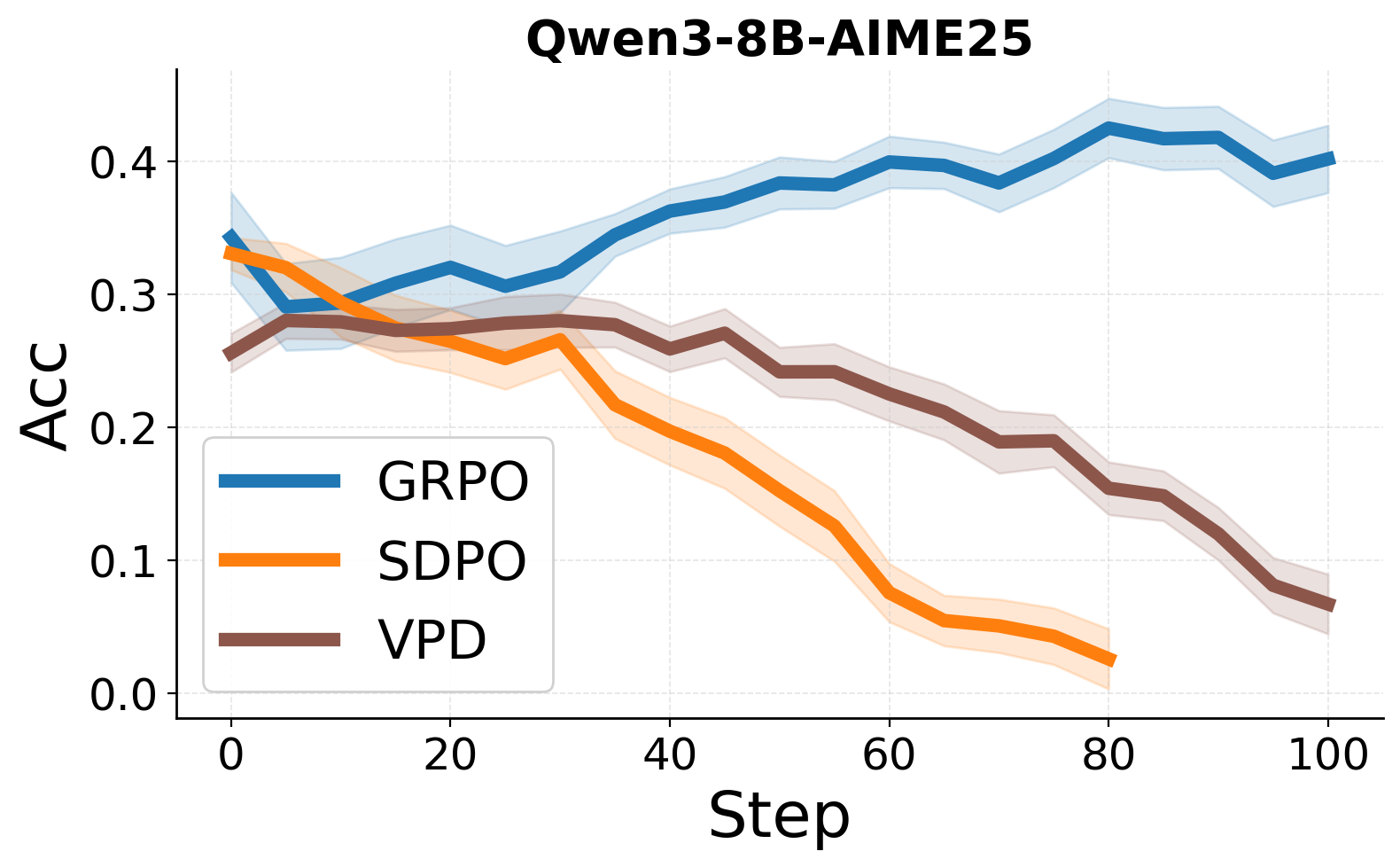}\hfill
    \includegraphics[width=0.32\linewidth]{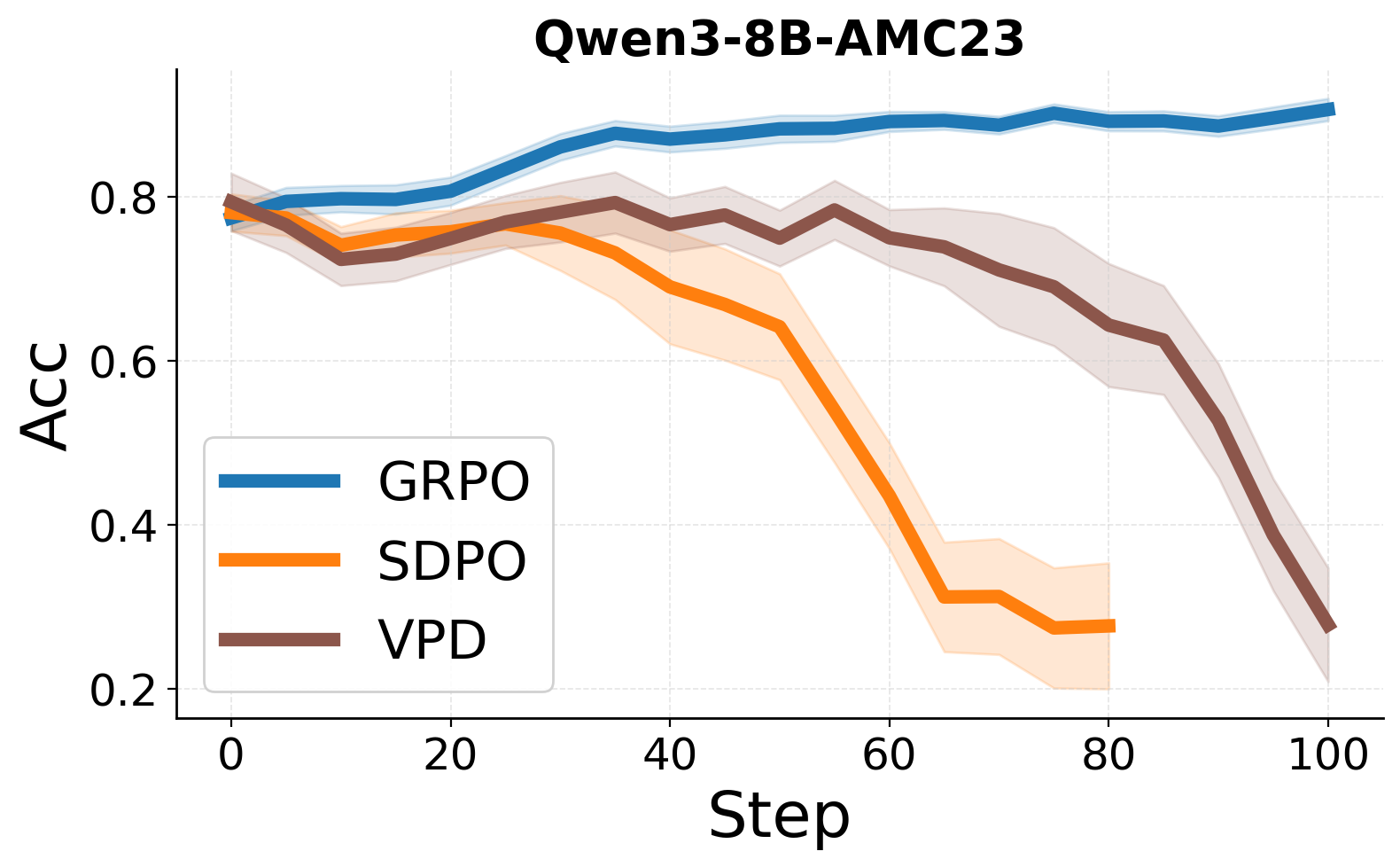}
    \caption{Additional training curves illustrating the performance on the AIME24, AIME25, and AMC23 benchmarks. Consistent with our Math500 findings, while VPD actively delays the severe training collapse observed with standard SDPO, pure sparse RL (GRPO) remains the most robust optimization method for these strict, highly symbolic domains.}
    \label{fig:math_app}
\end{figure}

\subsection{Ablations}\label{subsec:ablation}

\textbf{E-Step Update Frequency ($F$).}
To investigate the impact of the asymmetric update frequency, we utilize Qwen3-1.7B on the SciKnowEval benchmark. All foundational hyperparameters remain strictly identical to the configuration detailed in Table~\ref{tab:sciknoweval_hyperparams}. In our algorithmic implementation, configuring the update frequency to $F$ dictates the optimization schedule at the rollout-batch level. For every generated batch of on-policy rollouts, we perform the standard M-step distillation updates for the student. Additionally, once every $F$ rollout batches, we reuse this exact same batch of trajectories to perform an E-step optimization to refine the teacher.

As discussed in the main text, $F=5$ provides an optimal target-network dynamic. When setting $F=1$ (synchronous updates), the temporal runtime overhead increases as well, and the rapidly shifting teacher distribution prevents the student from adequately converging on the distillation targets. Conversely, for $F=10$, the teacher's diagnostic capabilities stagnate relative to the student's learning pace, yielding stale guidance that fails to appropriately address the student's newly emerging logical errors.

\textbf{Dynamic vs. Fixed Reference Prior.}
In standard reinforcement learning (e.g., PPO) and traditional self-distillation paradigms (e.g., SDPO), the policy is typically regularized via a KL-divergence penalty against a completely frozen reference model ($\pi_{\text{ref}}$), which is usually the initial base model. For our Fixed Prior ablation, we replicate this standard mechanism during the E-step, anchoring the teacher's BCO objective strictly to the frozen Qwen3-1.7B initialization. For our standard VPD (Dynamic Prior) setting, this prior is instead continually updated to match the active student policy ($\pi_\theta$) at the start of each expectation-maximization cycle.

The severe training instability observed with the Fixed Prior (illustrated in Figure~\ref{fig:prior}) stems from an escalating distribution shift. As the student policy naturally evolves and discovers novel, successful reasoning paths during the iterative M-steps, its distribution inherently drifts from the initial base model. If the teacher is subsequently forced to anchor its E-step updates to the frozen $\pi_{\text{ref}}$, its optimized target distribution is artificially pulled backwards toward the outdated base behaviors. This creates a destructive gradient conflict during the subsequent M-step: the student is penalized for exploring new reasoning paths and is instead distilled toward stale, pre-training priors. By dynamically anchoring the prior to $\pi_\theta$, VPD guarantees a sliding trust region, ensuring the teacher's guidance smoothly tracks the student's active exploration space without inducing training collapse. Figure~\ref{fig:prior_app} provides additional training progression plots across the remaining SciKnowEval domains.

\begin{figure}[h]
    \centering
    \includegraphics[width=0.3\linewidth]{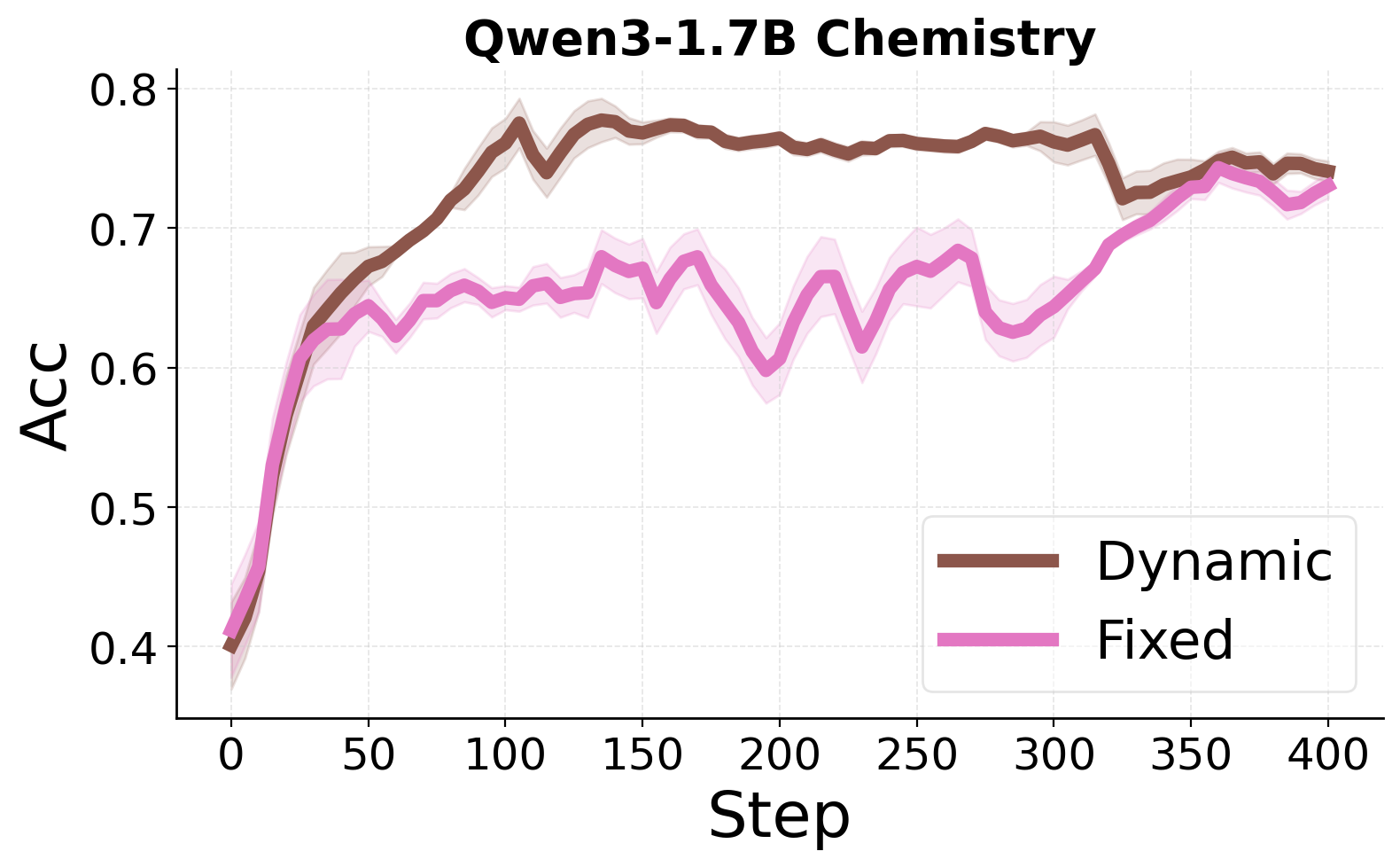}
    \includegraphics[width=0.3\linewidth]{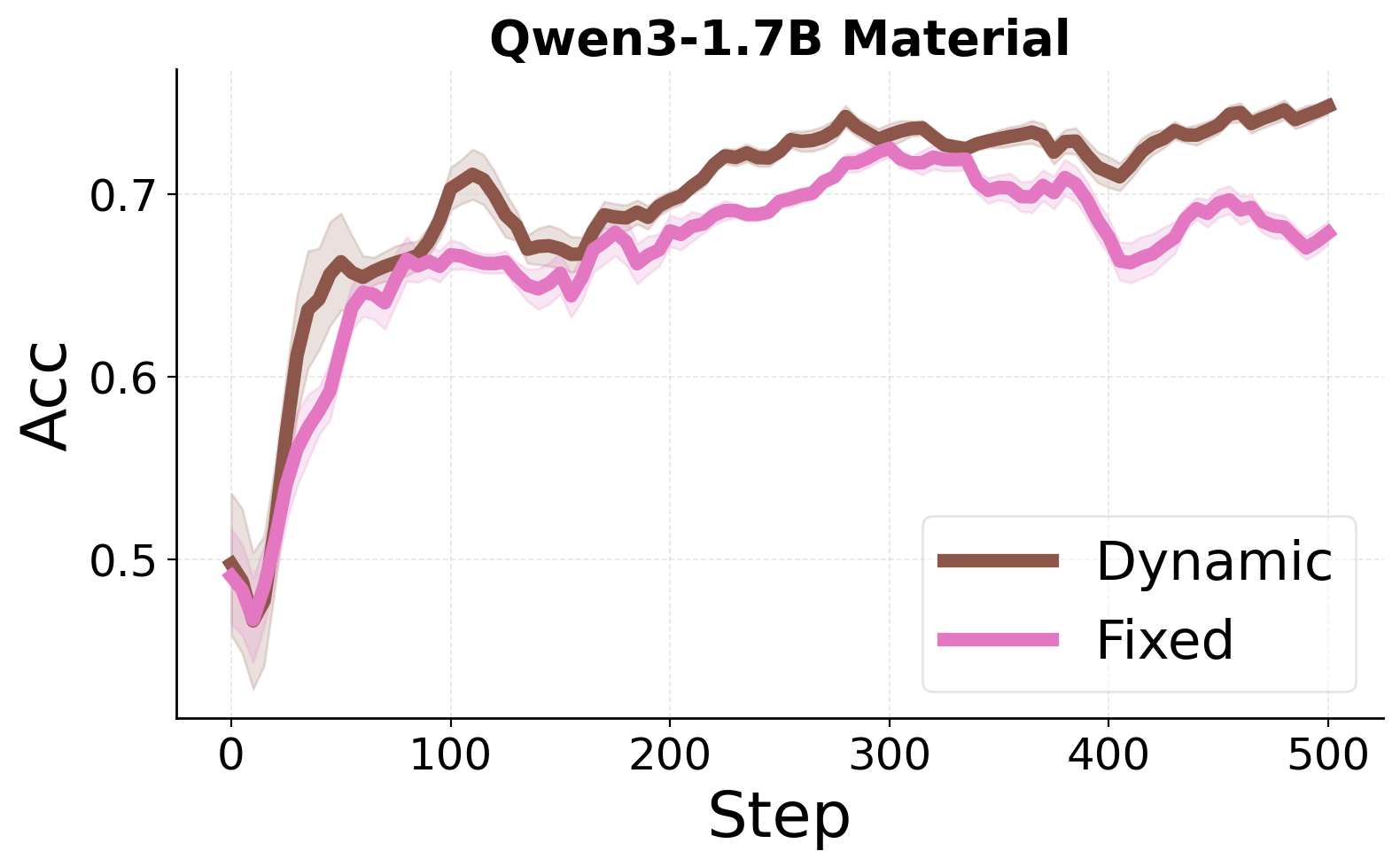}
    \includegraphics[width=0.3\linewidth]{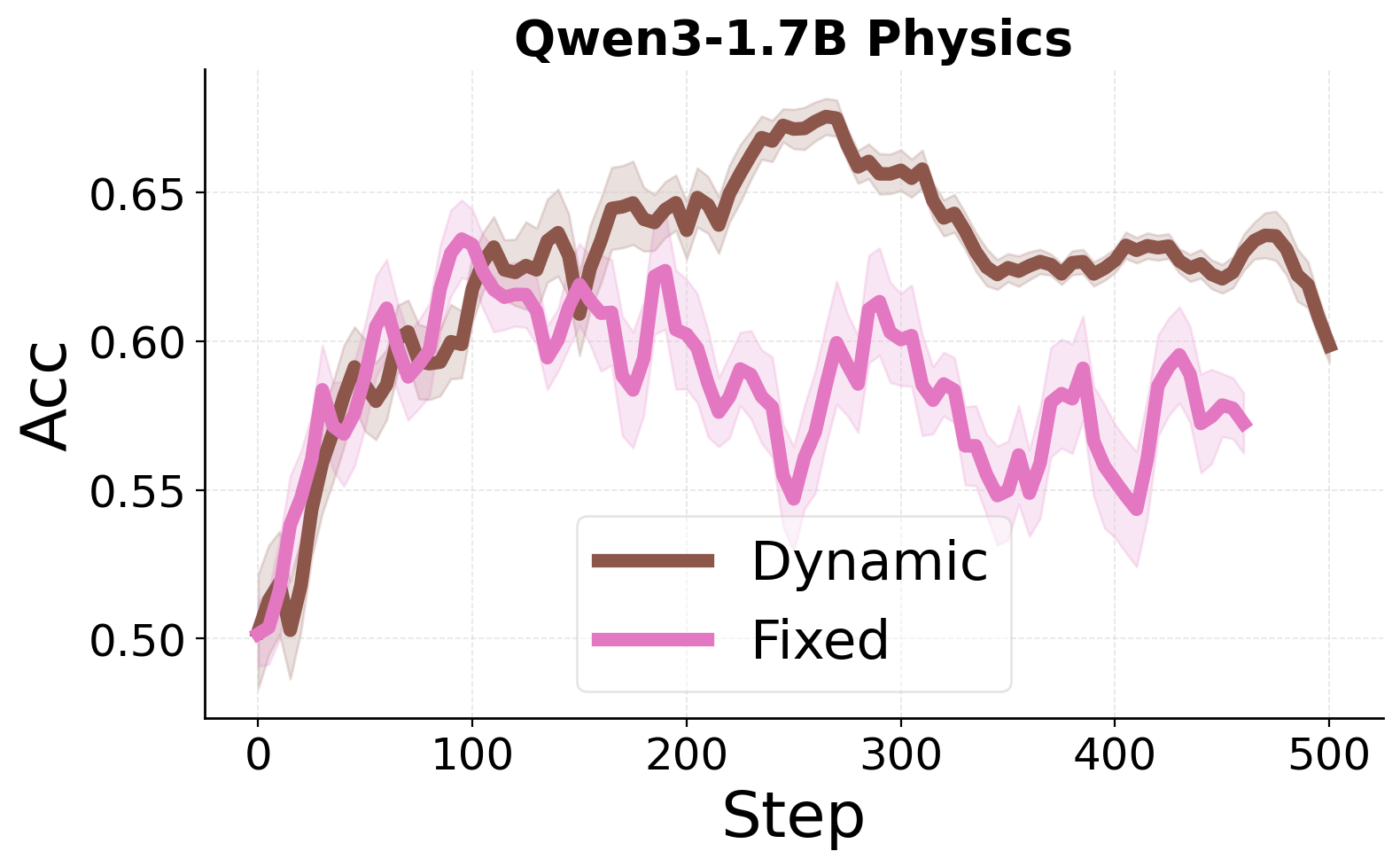}
\caption{Additional training curves comparing the dynamic and fixed reference priors during E-step optimization across the Chemistry, Materials Science, and Physics domains.}    
\label{fig:prior_app}
\end{figure}



\end{document}